\documentclass{article}

\PassOptionsToPackage{sort&compress, numbers}{natbib}
\usepackage[preprint]{neurips_2026}

\usepackage[utf8]{inputenc} 
\usepackage[T1]{fontenc}    
\usepackage{hyperref}       
\usepackage{url}            
\usepackage{booktabs}       
\usepackage{amsfonts}       
\usepackage{nicefrac}       
\usepackage{microtype}      
\usepackage{xcolor}         
\usepackage{amsmath,amssymb}
\usepackage[nolist]{acronym}
\usepackage{mathtools}
\usepackage{subcaption}
\usepackage{graphicx}
\usepackage[disable]{todonotes}
\usepackage{bm}
\usepackage{amsmath,amssymb,amsthm}
\newtheorem{theorem}{Theorem}

\usepackage{listings}
\newcommand{\G}{\ensuremath{\mathcal{G}}}

\newcommand{\ii}{\ensuremath{\mathrm{i}}}
\newcommand{\so}{\ensuremath{\mathrm{SO}(2)}}
\newcommand{\se}{\ensuremath{\mathrm{SE}(2)}}
\newcommand{\oo}{\ensuremath{\mathrm{O}(2)}}

\DeclareMathOperator{\BigO}{O}

\title{On Making $\bm{\mathrm{SE}(2)}$-Invariant Networks Optimal}

\author{%
  Tomas Karella\thanks{Corresponding author:
    \texttt{karella@lanl.gov}} \\
  CAI-3: Information Sciences \\
  Los Alamos National Laboratory
  \And
  Emily Shinkle \\
  CAI-3: Information Sciences \\
  Los Alamos National Laboratory
  \AND
  Alice Elisabeth Anastasia Allen\\
  Machine Learning in Molecular Dynamics\\ 
  Max Planck Institute for Polymer Research 
  \And 
  Pieter Swart \\
  T-5: Applied Mathematics 
  and Plasma Physics \\
  Los Alamos National Laboratory
  \And
  Nicholas Lubbers \\
  CAI-3: Information Sciences \\
  Los Alamos National Laboratory
  \And
  Roxana Bujack \\
  CAI-3: Information Sciences \\
  Los Alamos National Laboratory
}
\begin{acronym}
\acro{CNN}{convolutional neural network}
\acro{ViT}{Vision Transformer}
\acro{DoF}{degrees of freedom}
\acro{RCI}{rotation consistency index}
\end{acronym}

\begin{document}

\maketitle

\todo[]{Remove flips,  S from SE, SO}

\begin{abstract}
Respecting a-priori known symmetries of underlying data is a key principle in designing efficient neural architectures. Convolutional neural networks are a canonical example, encoding translation equivariance. Extending this inductive bias to richer symmetry groups such as $\mathrm{SE}(2)$ in the most effective way possible is an ongoing architectural and computational challenge. Building on classical invariant theory, we derive invariant feature maps that are provably optimal: minimal and information-preserving with respect to the symmetry group. Our construction is designed to integrate seamlessly into standard architectures, requiring no modification to nonlinearities or normalization layers. As a concrete implementation, we introduce a plug-and-play module for convolutional neural networks that enforces $\mathrm{SE}(2)$ invariance, capturing both translations and rotations within standard architectures. The resulting models are both efficient and expressive. Experiments show competitive or improved accuracy compared to state-of-the-art equivariant methods at significantly lower computational cost. Finally, we provide a theoretical and empirical analysis of the relationship between equivariance and invariance, offering insight into their roles in controlling robustness and flexibility.
\end{abstract}

\Acp{CNN}~\citep{lecun1989} remain a fundamental building block in computer vision, largely due to their strong inductive biases~\citep{prince2023}. In particular, translation equivariance enables consistent responses to spatial shifts and has proven crucial in data-limited regimes~\cite{bronstein2021}. More generally, incorporating symmetry priors into neural architectures has emerged as a principled strategy for improving sample efficiency and robustness, especially in scientific and engineering applications where underlying symmetries are well understood~\citep{bronstein2021, brehmer2025}. This has motivated a growing body of work extending equivariance beyond translations to other groups, such as rotations and reflections, to capture richer geometric structure in underlying data~\citep{cohen2016, kondor2018, worrall2019, fuchs2020, brandstetter2022, musaelian2023, worrall2017, equiformerv3, equiformerv2, equiformerv1, e2cnn}.

Many approaches to rotation equivariance construct feature spaces that transform according to irreducible representations of the underlying symmetry group~\citep{cesa2022, e2cnn, equiformerv1, worrall2017, weiler2018}. While this provides a principled framework, it introduces significant computational and architectural overhead. In particular, operations on such structured features often scale quadratically with the number of basis functions used to parameterize the feature space~\citep{cesa2022, e2cnn, worrall2017}, leading to increased complexity. Furthermore, standard deep learning components, such as convolutions, nonlinearities, and normalization layers, cannot be applied directly and must be redesigned to operate on these structured feature spaces. This complicates implementation and limits compatibility with existing architectures. Consequently, there remains a need for approaches that retain the benefits of symmetry-aware representations while enabling efficient~\citep{equiformerv3, passaro2023} and plug-and-play integration into standard neural network pipelines~\citep{kaba2023}.

An alternative framework for handling symmetry comes from classical pattern recognition~\citep{flusser2009, pratt2013, gonzalez2009}. In this approach, invariance is achieved by constructing features that are invariant to the transformation group~\citep{flusser2009}, avoiding the need for structured intermediate descriptors and instead operating directly in the invariant space. Early work~\citep{wood1996} highlighted that invariance alone is not sufficient: features must remain \textbf{discriminative}, i.e., preserve the ability to distinguish patterns up to the group action, avoid redundancy, and remain computationally tractable. Further, Flusser constructed invariant features that are both independent and discriminative on subsets of the input space where the group acts freely (e.g., under conditions such as $m_{1,0} \neq 0$)~\citep{flusser2000, flusser2002}. More recent work generalizes this result by constructing discriminative invariant maps on broader subsets of the input space~\citep{bujack2022}.

Our contribution is fourfold. First, we derive an \textbf{optimal invariant map} grounded in classical invariant theory and prove that it is both \textbf{discriminative}---preserving all information up to the group action---and \textbf{minimal}. Second, we translate this theory directly to an efficient implementation as a \textbf{plug-and-play module} yielding $\mathrm{SE}(2)$-invariant \acp{CNN} requiring no modification to standard components and yielding models that are both computationally efficient and expressive. Third, we show that these models achieve \textbf{competitive or improved performance} compared to equivariant approaches at lower computational cost on a variety of image datasets. Fourth, we provide a \textbf{theoretical and empirical analysis} of the relationship between equivariance and invariance, offering insight in balancing robustness and flexibility.

\section{Background: Discriminative and Minimal Invariant Maps}
We first review the invariant-theoretic framework underlying our approach. We begin by describing moments as projections onto basis functions with well-defined transformation properties, and then show how such so-called equivariant features can be systematically combined into invariants. This formulation allows for the development of feature maps that are both discriminative and minimal, forming the basis for the neural modules introduced in this work.

\subsection{Complex Moments and Irreducible Representations of $\texorpdfstring{\bm{\mathrm{SO}}\bm{(2)}}{SO(2)}$}
In pattern recognition, the goal is to distinguish between patterns $p \in \mathcal{P}$ based on observed data, where pattern $p$ is a function defined on compact support $p :\Omega \subset \mathbb{R}^d \to \mathbb{K}$. When symmetries are present, they can be modeled by a group $\mathcal{G}$ acting on $\mathcal{P}$, which defines how patterns \(p\) behave under transformation. Two fundamental behaviors arise.

A function $\hat{\phi}: \mathcal{P} \rightarrow \mathbb{K}$ is \textbf{equivariant} if it transforms consistently with the group action. 
Typically, the action of $\mathcal{G}$ is realized through representations, so that each $g \in \mathcal{G}$ is associated with a matrix $\rho$ acting on the corresponding space. The group actions on the input and output spaces may be given by different representations, and equivariance is expressed as
\begin{equation}
\label{def:equivariance}
\hat{\phi}(\rho_{\mathrm{in}}(g)\, p)
=
\rho_{\mathrm{out}}(g)\, \hat{\phi}(p),
\quad \forall g \in \G, \, p \in \mathcal{P}.
\end{equation}
and \textbf{invariant} if it remains unchanged under the group action, i.e.
$\phi(\rho_{\mathrm{in}}(g)\, p) = \phi(p)$.

To construct informative features of patterns, it is common to project signals onto a set of basis functions with well-defined $\G$ transformation properties. In the context of group symmetries, such bases are often chosen to align with irreducible representations~\citep{weiler2023} of the underlying group, as formalized by results such as the Peter-Weyl Theorem~\citep{Peter1927} for compact groups. Classically, these projections are referred to as \textbf{moments}~\citep{flusser2009}.

Let \(p : \Omega \subset \mathbb{R}^d \to \mathbb{K}\) be a pattern with compact support, and let \(\psi_{k,l} : \Omega \to \mathbb{K}\) be a basis function. The \textbf{moment} of \(p\) associated with \(\psi_{k,l}\) is defined as
\begin{equation}
\label{def:moment}
m[\psi_{k,l}](p) := \int_{\Omega} \psi_{k,l}(x)\, p(x)\, dx .
\end{equation}
Moments are defined with respect to a chosen family of basis functions. For rotational symmetries, widely used basis functions---such as complex or Zernike polynomials~\citep{flusser2009}, or equivariant steerable bases~\citep{weiler2018, cesa2022, e2cnn}---share a common structure: they decompose into radial and angular components, differing in the radial term. We therefore consider a general family of basis functions of the form 
\begin{equation}
\label{def:polynomial}
\psi_{k,l}(r,\theta) = R_{k,l}(r)\, e^{\ii k \theta},
\end{equation}
where \((r,\theta)\) are polar coordinates, \(R_{k,l}\) is a radial function, \(k \in \mathbb{Z}\) denotes the angular frequency, and \(l\) indexes the radial component. Since \(e^{\ii k \theta}\) transforms according to an irreducible representation of \(\mathrm{SO}(2)\), each basis function \(\psi_{k,l}\) induces a moment \(m[\psi_{k,l}]\) with a well-defined type, also called spin, weight, character, or eigenvalue, given by \(k\). For notational simplicity, we write \(m_{k,l} := m[\psi_{k,l}]\). When the radial index is fixed or irrelevant, we simply write \(m_k\). Formally, for a rotation \(g_\theta \in \mathrm{SO}(2)\) by angle \(\theta\), a moment of type \(k\) satisfies
\begin{equation}
\label{eq:type}
m_k(g_\theta \odot p) \overset{\eqref{def:moment}\eqref{def:polynomial}}{=} e^{\ii k \theta} \, m_k(p).
\end{equation}
\textbf{Rotation invariants}~\citep{flusser2009} can be constructed by combining moments in such a way that their transformation phases cancel. In particular, consider a product of moments
\begin{equation}
\label{eq:invariant}
\phi(p) := \prod_{j=1}^n m_{k_j}(p)^{a_j}, \qquad \text{subject to} \sum_{j=1}^n k_j a_j = 0 .
\end{equation}
then $\phi$ is invariant under rotation $g_\theta$. Indeed, using \eqref{eq:type} and \eqref{eq:invariant} and rotation $g_\theta$, we have 
\begin{equation*}
\phi(g_\theta \odot p)
\overset{\eqref{eq:type}\eqref{eq:invariant}}{=}
\prod_{j=1}^n \left(e^{\ii k_j \theta} m_{k_j}(p)\right)^{a_j} 
e^{\ii \theta \sum_{j=1}^n k_j a_j} \, \phi(p) \overset{\eqref{eq:invariant}}{=} \phi(p).
\end{equation*}

\subsection{Flusser Basis as a Minimal and Discriminative Invariant Map}
Equation~\ref{eq:invariant} allows for an infinite number of invariants that emerge as a combinatorial explosion. Early approaches, such as Hu’s\citep{hu1962} moment invariants, proposed specific targeted subsets of such combinations, but these were proven redundant and did not capture all relevant information. This limitation was addressed by Flusser~\citep{flusser2002}, who proposed a systematic construction of non-redundant invariant representations, known as the \textbf{Flusser basis}. The key idea is to normalize moments of arbitrary type \(k\) by a reference moment of type \(1\), thereby canceling the rotational phase. This yields the invariant map
\begin{equation}
\label{eq:flusser_basis}
\Phi_F(p) :=
\left\{
m_{k,l}(p)\, m_{1,0}(p)^{-k}
\;\middle|\;
1 \leq k \leq N_k,\;
1 \leq l \leq N_l,\;
m_{1,0}(p) \neq 0
\right\},
\end{equation}
where \(m_{1,0}\) is a non-zero reference moment of type \(1\) and \(N_k, N_l \in \mathbb{N}^+\) control the maximum angular and radial frequencies. This construction yields a non-redundant set of invariants and is discriminative up to rotation. It can be interpreted as a canonical phase normalization of moments. Our formulation generalizes classical constructions: complex moments~\citep{flusser2000} are typically indexed as \(m_{pq}\) with type \(k = p - q\), which corresponds to our notation \(m_{k,l}\) up to the choice of radial functions.

The resulting Flusser basis has dimension \(\dim(\Phi_F) = \dim(\mathcal{M}) - \dim(G)\) (cf.\ Appendix~\ref{appendix:dimensions}), indicating that it captures all degrees of freedom of the signal up to the group action. Here, \(\mathcal{M}\) denotes the manifold induced by the moments.\footnote{In cases where $(m_1 = 0)$, the flexible invariant basis~\citep{bujack2017} can be employed; see Appendix~\ref{appendix:flexible} for details.} 

To formalize properties of \textbf{optimal bases} in our setting, we revisit and adapt the notions of minimality and discriminability from classical pattern recognition~\citep{wood1996}. Let \(\Phi : \mathcal{P} \to \mathbb{K}^n\) denote a map of $n$ invariant features, with components \(\phi_j : \mathcal{P} \to \mathbb{K}\). In the context of deep learning, where networks can approximate arbitrary functions~\citep{hornik1989}, invariant maps should be non-redundant, i.e., minimal, meaning that the features do not admit functional dependencies.

An invariant map \(\Phi\) is \textbf{discriminative} if, for all \(p_1, p_2 \in \mathcal{P}\),
\begin{equation}
\Phi(p_1) = \Phi(p_2) \quad \Rightarrow \quad \exists\, g \in \mathcal{G} \text{ such that } p_2 = g \odot p_1.
\end{equation}

An invariant map $\Phi = (\phi_1,\dots,\phi_m)$ is called \textbf{minimal} if it is discriminative and no proper subcollection is discriminative, i.e.,
\begin{equation}
\forall\, \Psi \subsetneq \Phi,\quad \Psi \text{ is not discriminative.}
\end{equation}
Equivalently, no component $\phi_i$ can be removed without losing discriminability. Functional independence is sufficient for minimality~\citep{flusser2000}, while discriminability for patterns with $m_{1,0} \neq 0$ follows from completeness and the solution of the inverse problem~\citep{flusser2002}.

\section{Building $\texorpdfstring{\bm{\mathrm{SE}}\bm{(2)}}{SE(2)}$-Invariant Networks with Moment Invariants}
Having established the invariant theory, we now introduce our method. We construct $\se$-invariant \acp{CNN} using optimal bases derived from moment invariants. We consider a setting in which $\se$ acts only on the spatial domain, while feature values remain invariant. In contrast to classical approaches, where invariants are computed from a single pattern $p$ (e.g., at the center of mass), we evaluate invariants at every spatial location of the input feature space $f : \Omega \subset \mathbb{R}^2 \to \mathbb{R}$. Accordingly, each layer of the network $\hat{\Phi}$ is required to satisfy the equivariance constraint on the spatial domain, while the output feature values themselves remain invariant under rotations ($\rho_\textrm{out}=\mathbb{I}$):
\begin{equation}
\label{eq:se2_equivariance}
\hat{\Phi}(g \odot f)(x) = \hat{\Phi}(f)(g^{-1} x), \quad g \in \se,
\end{equation}
where $g = (g_t, g_\theta) \in \se$ denote translation and rotation, and $\odot$ is the action of $\mathrm{SE}(2)$ on the input signal.

Many standard operations used in deep \acp{CNN} already satisfy Eq.~\eqref{eq:se2_equivariance} (up to discretizations errors), including point-wise nonlinearities and normalization layers, as they act independently at each pixel. In the context of \acp{CNN}, standard convolutional layers are not equivariant under rotations and therefore must be redesigned to preserve the desired transformation behavior under $\se$.

\subsection{Moment Layer: Projection to a basis}
The proposed moment projection layer can be viewed as a form of steerable convolution~\citep{cesa2022}. In steerable \acp{CNN}~\citep{worrall2017,e2cnn,cesa2022,weiler2018}, equivariance is enforced by constraining convolutional kernels to transform according to group representations. Our construction follows the same principle: we project the input onto a moment basis at each spatial location. This maps local signals into a moment space whose components transform according to their type. Concretely, each moment $m_k$ satisfies
\begin{equation}
\label{eq:steerable_field}
m_k(g \odot f)(x) \overset{\eqref{eq:type}}{=} e^{\ii k \theta} m_k(f)(g^{-1} x), \quad g = (g_t, g_\theta) \in \se,
\end{equation}
so that the feature map $(m_{k_1}, \dots, m_{k_n})$ forms a steerable field transforming under the direct sum of irreducible representations, that is, the combined representation is given by $\rho(g_\theta) = \bigoplus_j e^{\ii k_j \theta}$.

Steerable feature fields provide a principled way to achieve equivariance and have been studied in prior work through specialized architectures that preserve the transformation law in Eq.~\eqref{eq:steerable_field}~\citep{weiler2023}. In these approaches, all layers must be adapted accordingly, and convolutional layers explicitly model interactions between all input and output types. As implemented in frameworks such as escnn~\citep{e2cnn, cesa2022}, steerable convolutions incur a computational scaling of $\BigO(N_{\text{in}}N_{\text{out}})$ due to these type-to-type interactions.

In contrast, our construction avoids explicit type-to-type interactions by directly projecting all features back into type $0$, the invariant type. Each block consists of (i) a moment projection layer and (ii) an invariant layer that combines them into a complete and non-redundant set of invariant features over the quotient space $\mathcal{M}/G$. Our block scales linearly with the number of moments.

\paragraph{Implementation.} The moment projection is implemented as a convolution, computing moments at each spatial location. The kernels $\psi_{k,l}(r,\theta)$ form a separable basis in polar coordinates (Eq.~\eqref{def:polynomial}), with angular frequency $k$ and radial components $R_{k,l}$. In practice, $R_{k,l}$ can be fixed or learned; following prior work~\citep{cesa2022,e2cnn,worrall2017}, we use a B-spline parameterization. Similar to prior symmetry-aware architectures~\citep{e2cnn,mnist_rot_test}, we also employ larger \ac{CNN} kernels than in standard \acp{CNN} to mitigate discretizations error. Additional details on radial parameterizations and kernel sizes are provided in Appendix~\ref{appendix:radial}. For stability, moment features are normalized using a shared layer normalization across types, similarly to Equiformer~\citep{equiformerv3}. We observe that performance is not sensitive to the specific normalization choice.

\subsection{Optimal Bases: Preserving All Information from Moments}
A natural approach to constructing invariant features is to directly apply classical moment invariants, such as the Flusser basis defined in Eq.~\eqref{eq:flusser_basis}, which combine moments into type-$0$ quantities. In our setting, these constructions are applied locally at each spatial location. Concretely, one can define pointwise invariants of the form
\begin{equation}
\label{eq:flusser_conv}
\phi_k(f)(x) \overset{\eqref{eq:flusser_basis}}{=} m_k(f)(x) \left[m_1(f)(x)\right]^{-k},
\end{equation}
which cancels the phase and yields a type-$0$ quantity. For simplicity, we omit the explicit dependence on $f$ and $x$ when it is clear from context, and write $m_k$ and $\phi_k$ instead of $m_{k,l}(f)(x)$ and $\phi_k(f)(x)$.

Such formulations involve several complexities that need to be addressed for a practical implementation. First, they involve high-order polynomial combinations of moments, resulting in rapidly growing magnitudes and numerical instabilities. They also require explicit manipulation of complex-valued quantities, leading to inefficiency. Our concrete construction handles these difficulties.

\paragraph{Optimal Bases.}
We leverage the transformation law in Eq.~\eqref{eq:type} in two ways. (i) The \textbf{magnitudes $\boldsymbol{|m_k|}$ are rotation-invariant} and therefore provide type-$0$ features directly. (ii) Since the phase of each moment changes linearly with its type, relative phases can be normalized using a non-zero reference moment $m_1$, analogous to classical invariant constructions in Eq.~\eqref{eq:flusser_basis}. The resulting \textbf{phase differences can be represented via cosine or sine functions}. Together, these features capture the invariant degrees of freedom and are sufficient to reconstruct classical moment invariants (Appendix~\ref{appedix:proof}). We define the invariant feature basis $\Phi_{\mathcal{O}}$ as
\begin{equation}
\Phi_{\mathcal{O}}
=
\left\{
|m_1|,\;
|m_k|,\;
\phi_{\mathcal{O}}(m_1, m_k)
\;\middle|\;
k = 0,2 \dots N_k \
\right\},
\end{equation}
where the magnitude terms provide rotation-invariant amplitudes and $\phi_{\mathcal{O}}(m_1, m_k)$ encodes rotation-invariant relative phase between moments. 

To make the phase construction explicit, we express moments in polar form as $m_k = r_k e^{i\theta_k}$. As mentioned, direct constructions such as in Eq.~\eqref{eq:flusser_conv} introduce powers of magnitudes, leading to value explosion in deep layers. Although magnitude information is already captured, the phase alone cannot be used directly, as phases of low-magnitude moments are unstable: the phase becomes undefined as the magnitude approaches zero. We therefore construct a single real-valued invariant per pair $(m_1, m_k)$ by taking the real part of the phase-aligned interaction and modulating it by a magnitude-dependent function:
\begin{equation}
\label{eq:our-basis}
\phi_{\mathcal{O}}(m_1, m_k)
=
\sigma(|m_1|, |m_k|)\,
\cos\!\big(\theta_k - k\theta_1\big).\footnote{This follows from complex multiplication in polar form, $c_1 c_2 = r_1 r_2 (\cos(\theta_1+\theta_2) + i\sin(\theta_1+\theta_2))$, and de Moivre's formula~\citep{moivre1707}, $c^w = r^w (\cos(w\theta) + i\sin(w\theta))$. We retain only the real part.}
\end{equation}
This suppresses unstable phase contributions near zero-magnitude patterns while preventing large polynomial scaling with the magnitude of the pattern.

\begin{figure}
    \centering
    \caption{Overview of the proposed invariant convolutional architecture. (a) The overall network architecture obtained by stacking the proposed blocks. (b) The invariant convolution block, combining a moment layer and an invariant layer within a residual structure.   (c) Illustration of the invariant layer, showing how moment responses are combined to produce rotation-invariant features.}
    \includegraphics[width=0.85\linewidth]{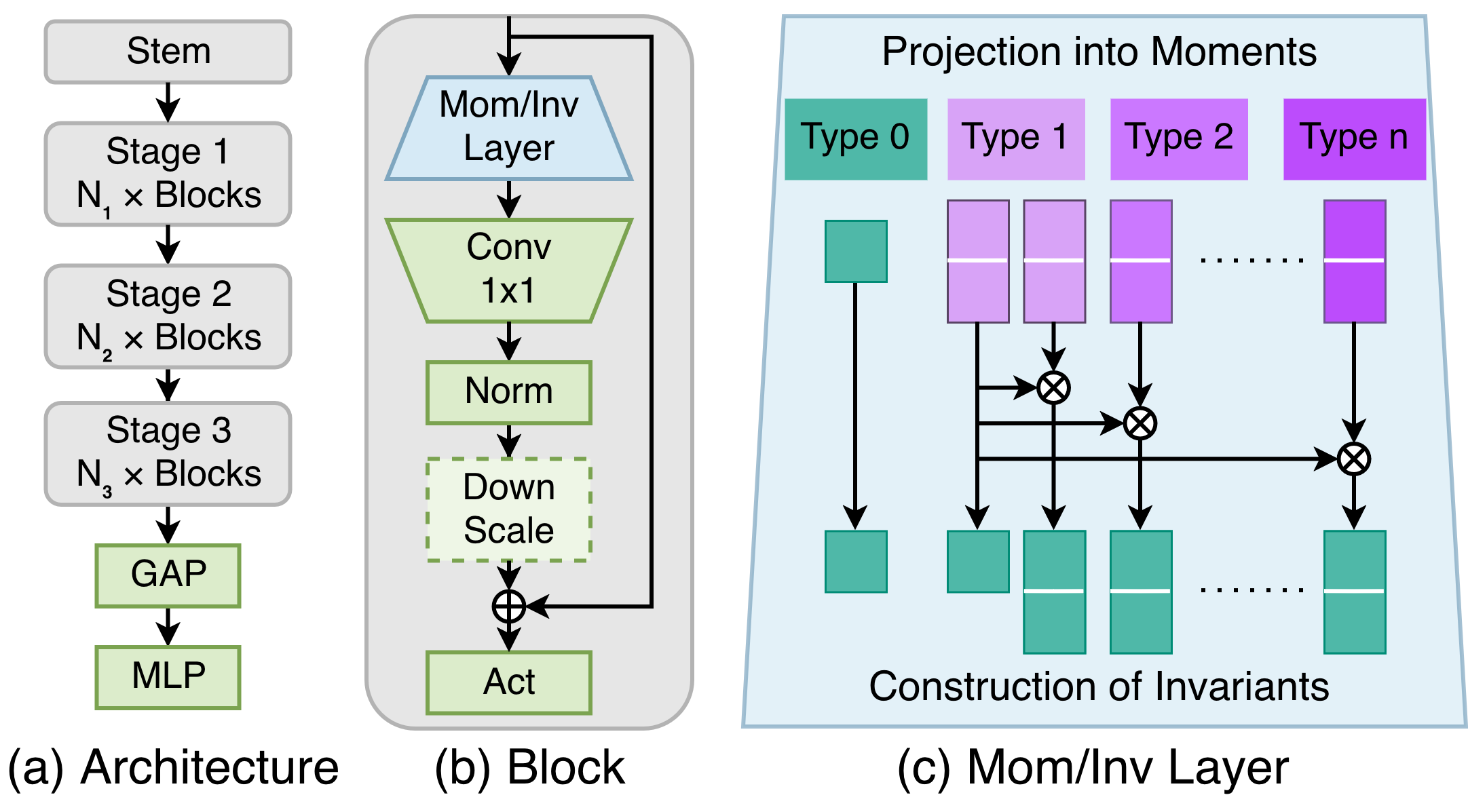}
    \label{fig:architecture}
    \vspace{-.1in}
\end{figure}

\subsection{Plug-and-Play Invariant Moment Layer}
The overall architecture of the \ac{CNN} block is depicted in Fig.~\ref{fig:architecture}. The moment projection is followed by an invariant layer that maps moment features to type-$0$ representations according to Eq.~\eqref{eq:our-basis}, combining moment magnitudes and relative phase differences with respect to a reference moment illustrated by Fig.~\ref{fig:architecture}(c). Together, these components form an invariant convolutional (Mom/Inv) layer that can be integrated into standard \ac{CNN} architectures.

The output invariant features can be aggregated in a variety of ways. In our implementation, we follow the design illustrated in Fig.~\ref{fig:architecture}(b), using a $1\times1$ convolution followed by normalization and a nonlinearity in a manner similar to MBConv blocks~\citep{tan2021}. The proposed blocks can be stacked to form standard \ac{CNN} architectures; in particular, we adopt a stage-wise design, where multiple blocks are grouped into stages operating at decreasing spatial resolutions and increasing channel dimensions, as illustrated in Fig.~\ref{fig:architecture}(a).

We compared the Mom/Inv layer design with alternative invariant maps, including Flusser and the flexible basis (Appendix~\ref{appendix:flexible}). While these approaches are theoretically equivalent in terms of discriminability, the proposed basis proved more numerically stable and computationally efficient in practice. Additional nonlinearities after the invariant projection did not yield consistent improvements, indicating that the construction already provides sufficient nonlinearity, similar to tensor product interactions in equivariant networks such as e3nn~\citep{equiformerv1}. We also considered normalization before or after the invariant layer; as both performed similarly, we apply it only after moment projection.

\section{Experiments}
We evaluate our method on benchmarks where rotational symmetries are intrinsic to the data, including satellite imagery, histology, material science, and textures (Fig.~\ref{fig:datasets}). In these domains, object orientation is arbitrary or physically irrelevant, providing a natural setting to assess whether rotational invariance improves generalization and data efficiency, particularly in low-data and capacity-constrained regimes. To further study this setting, we construct a material science benchmark based on SEM microscopy images, representing a distinct modality with limited data.

\paragraph{Metrics.}
In addition to classification accuracy, we measure the \ac{RCI} by measuring the agreement between predictions for rotated versions of the same input. Given logits for an upright input and its rotated counterpart, we compute the $\ell_2$ distance:
\begin{equation}
\mathrm{RCI} = \| \mathbf{z}_{\text{rot}} - \mathbf{z}_{\text{upright}} \|_2.
\end{equation}
Lower $\mathrm{RCI}$ indicates greater consistency under rotation. We found this metric to be particularly informative when comparing invariant and equivariant models. We also considered Lie derivative equivariance metrics~\citep{gruver2023}, but found them less sensitive for comparing models within the class of equivariant architectures.

\begin{figure}
\centering
\begin{minipage}{0.48\textwidth}
  \centering
  \includegraphics[width=\linewidth]{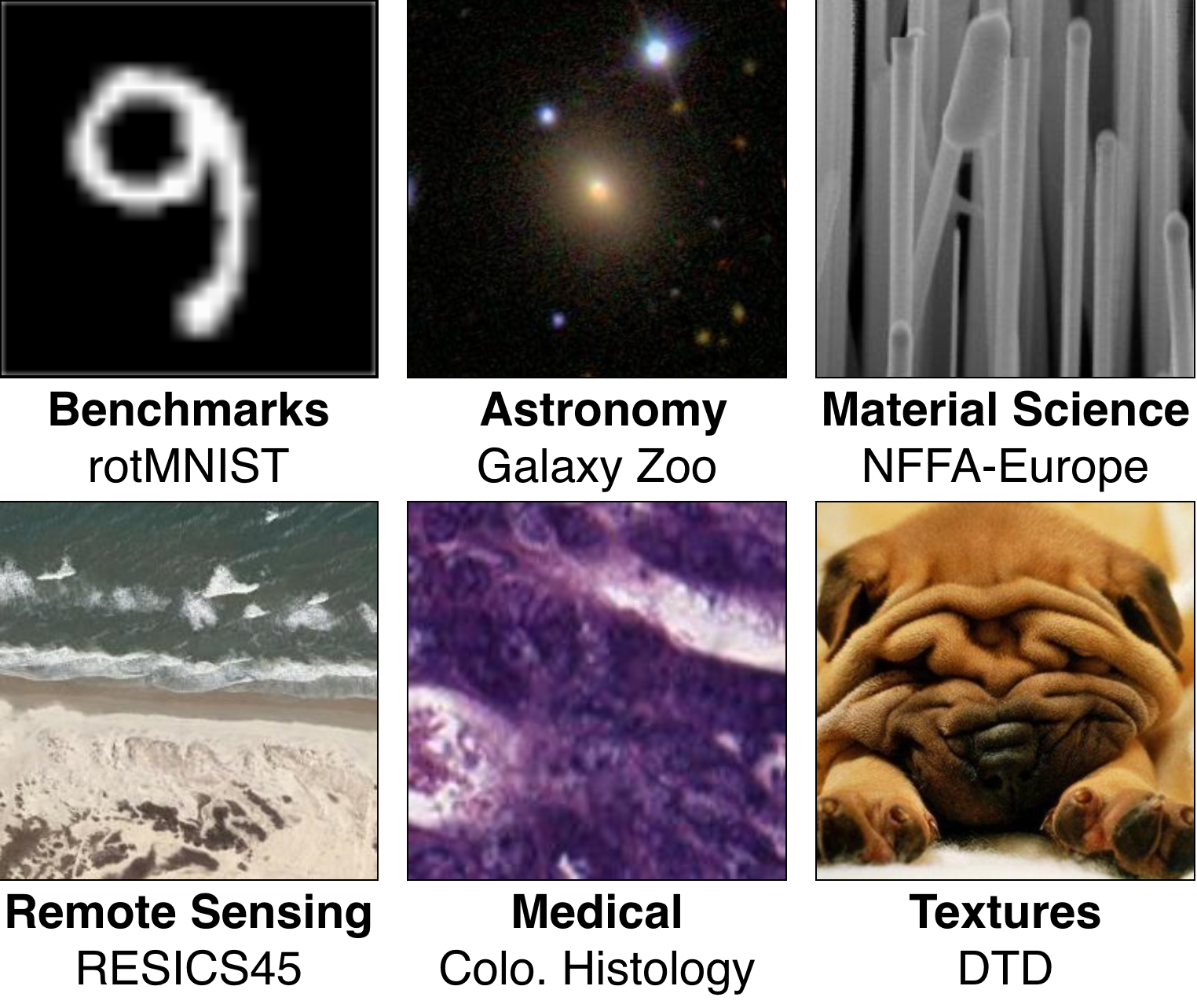}
  \caption{Benchmarks across multiple domains.}
  \label{fig:datasets}
\end{minipage}
\hfill
\begin{minipage}{0.49\textwidth}
  \centering
  \includegraphics[width=\linewidth]{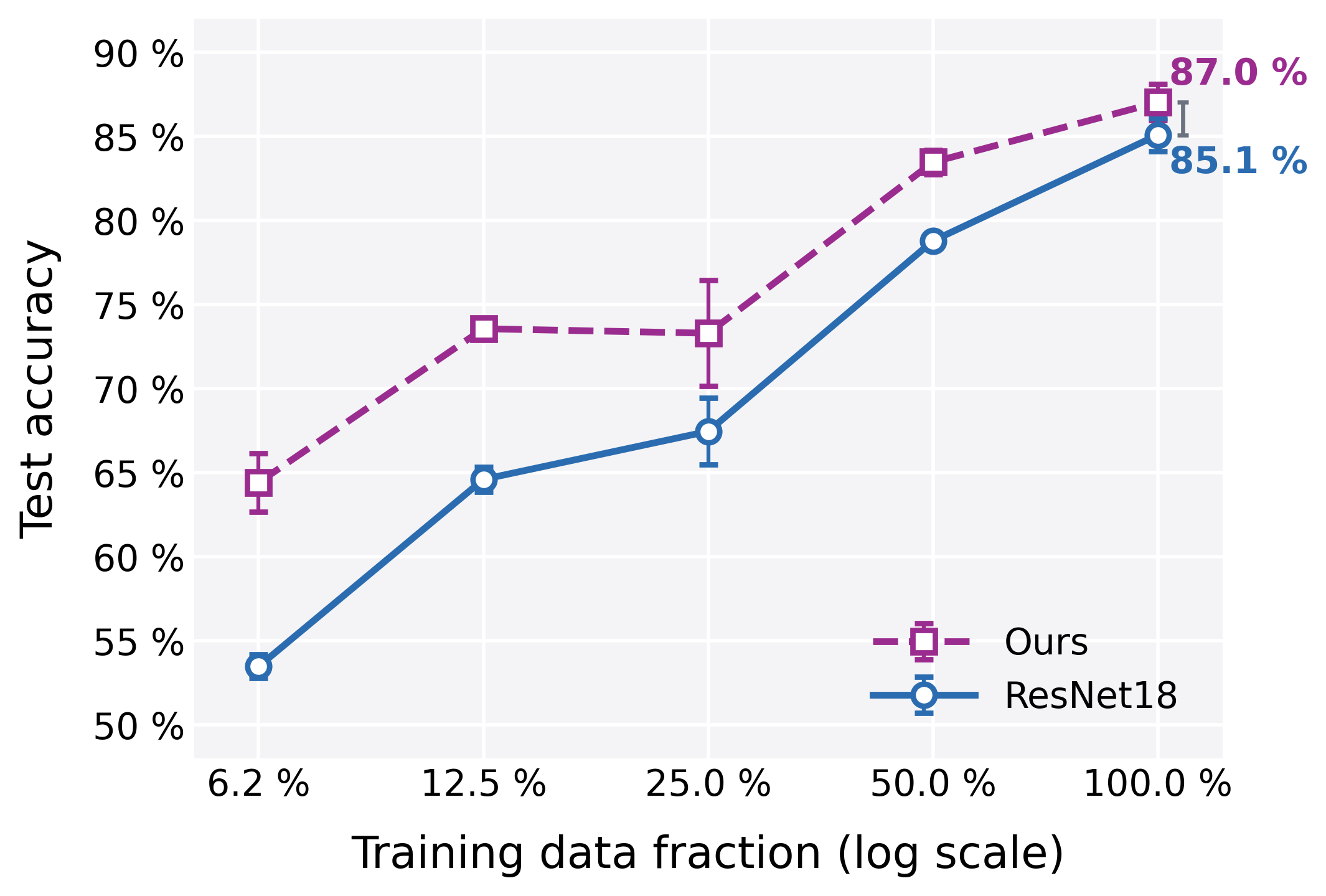}
  \caption{\textbf{Sample efficiency on NFFA.} Test accuracy vs.\ training data fraction. Our model (300k) compared with ResNet18 (11M, aug rot).}
  \label{fig:data-fraction}
\end{minipage}
\end{figure}

\paragraph{Datasets.}
We evaluate our model on a diverse set of datasets covering classical handwritten digits, remote sensing, textures, and scientific imaging. Specifically, we consider MNIST-Rot-Test~\citep{mnist_rot_test}, RotMNIST~\citep{rotmnist}, RESISC45~\citep{resisc45}, EuroSAT~\citep{eurosat}, DTD~\citep{dtd}, Lost Oil Wells~\citep{lost_oil_wells}, Galaxy Zoo~\citep{galaxy_zoo, galaxy_zoo1, galaxy_zoo2}, and Colorectal Histology~\citep{col_hist}. 

\paragraph{NFFA.} In addition, we construct a dataset based on NFFA-Europe~\citep{nffa_europe}, designed to reflect a realistic scientific imaging scenario where data is both scarce and structurally constrained. We manually curate labels and process the images (resizing and cropping to $224\times224$), resulting in a dataset with a small number of samples and high intra-class variability. This setting is particularly suitable for evaluating $\mathrm{SE}(2)$-invariant architectures in the low-data regime. To further examine this behavior, we analyze performance as a function of the available training data (Fig.~\ref{fig:data-fraction}), highlighting improvements in sample efficiency. Further details are provided in Appendix~\ref{appendix:nffa}.

\paragraph{Baselines.}
We compare against $E(2)$-CNNs with continuous $\mathrm{SO}(2)$ equivariance implemented via the \texttt{e2cnn} framework, a widely used and general implementation of steerable convolutional networks. Additional baselines, including H-Nets~\citep{worrall2017}, standard non-equivariant CNNs~\citep{he2016}, and transformer-based models~\citep{dosovitskiy2021} are reported in Appendix~\ref{appendix:quantitative}.

\subsection{Quantitative results}
Across all benchmarks, our approach achieves competitive or better accuracy compared to a baseline while using substantially fewer parameters, as shown in Table~\ref{tab:merged}. In particular, we observe consistent gains on standard rotation benchmarks, as well as on more complex datasets including RESISC45 and EuroSAT. Notably, these improvements are obtained with models that are typically two to three times smaller than their equivariant counterparts.

In addition to higher accuracy, our method yields significantly lower $\mathrm{RCI}$ values across all datasets, indicating stronger consistency under rotations. This difference is especially pronounced on benchmarks with explicit rotations, where our invariant representations reduce $\mathrm{RCI}$ by an order of magnitude, demonstrating a dramatic improvement in the invariance of the classification function. Even in cases where accuracy is comparable or slightly lower (e.g., Galaxy Zoo), our method maintains substantially lower $\mathrm{RCI}$ and reduced parameter count. We attribute this to the invariant construction in Eq.~\eqref{eq:our-basis}, which suppresses unstable phase variations by encoding relative phase differences. We also observe that larger convolutional kernels improve rotational consistency in our method, however, applying larger kernels in \texttt{e2cnn} reduced accuracy while further increasing computational cost and parameter count. Collectively, these results show that structured equivariant feature fields are not necessary complexities for invariant classification.

\begin{table}
\centering
\caption{\textbf{Quantitative comparison between our Mom/Inv \ac{CNN} and $E(2)$-CNNs.} Results are reported as mean $\pm$ standard deviation over multiple runs. For DTD, we report top-five accuracy, denoted by $^\ast$.
}
\label{tab:merged}
\begin{tabular}{lccc|ccc|ccc}
\toprule
& \multicolumn{3}{c|}{Dataset} 
& \multicolumn{3}{c|}{Ours} 
& \multicolumn{3}{c}{E2-CNNs} \\
\cmidrule(r){2-5} \cmidrule(r){5-7} \cmidrule(r){7-10}
Dataset 
& Reso.
& Cls 
& Size 
& Acc 
& Prm. 
& $\mathrm{RCI}$
& Acc 
& Prm. 
& $\mathrm{RCI}$
\\
\midrule
MNIST\tiny{RotTest} 
& $28$ 
& 10 
& 50k 
& $\mathbf{99.42}{\scriptstyle \pm.07}$ 
& 119k
& $\mathbf{0{\scriptstyle .01}}$
& $97.20{\scriptstyle \pm.14}$
& 362k 
& $0{\scriptstyle .89}$ \\

{\tiny rot}MNIST 
& $28$ 
& 10 
& 12/50k 
& $\mathbf{99.08}{\scriptstyle \pm .04} $ 
& 119k
& $\mathbf{0{\scriptstyle .01}}$
& $98.82{\scriptstyle \pm .01}$ 
& 362k 
& $0{\scriptstyle .64}$ \\

RESISC45 
& $256$ 
& 45 
& 31.5k 
& $\mathbf{85.77} {\scriptstyle \pm.47}$ 
&  230k 
&  $\mathbf{1{\scriptstyle .66}}$
& $84.02{\scriptstyle \pm.27}$ 
&  365k
&  $5{\scriptstyle .13}$ 
\\ 

EuroSAT\tiny{rgb}
& $64$ 
& 10 
& 27k 
& $\mathbf{95.29}{\scriptstyle \pm .40}$ 
& 149k 
& $\mathbf{0{\scriptstyle .71}}$
& $93.90{\scriptstyle \pm .34}$ 
& 363k 
& $1{\scriptstyle .60}$ \\


DTD$^\ast$  
& $\tiny{\leq}640$ 
& 47 
& 5k 
& $67.14{\scriptstyle \pm.27}$
& 336k
& $\mathbf{0{\scriptstyle .01}}$
& $67.18{\scriptstyle \pm .81}$ 
& 469k
& $1{\scriptstyle .97}$\\

NFFA
& $224$ 
& 10 
& 25k 
& $87.02 {\scriptstyle \pm.87}$
& 248k
& $\mathbf{0{\scriptstyle .01}}$
& $87.32 {\scriptstyle \pm.30}$
& 363k 
& $1{\scriptstyle .23}$ \\

Galaxy Zoo 
& $424$ 
& $8$ 
& 100/31k 
& $83.30{\scriptstyle \pm.03}$ 
& 227k
& $\mathbf{0{\scriptstyle .27}}$
& $\mathbf{84.39}{\scriptstyle \pm.08}$
& 363k 
& $0{\scriptstyle .82}$ \\

Col. Hist. 
& $150$ 
& 8 
& 5k
& $\mathbf{92.30}{\scriptstyle \pm.40}$
& 228k 
& $\mathbf{0{\scriptstyle.59}}$
& $91.58{\scriptstyle \pm .41}$
& 363k
& $1{\scriptstyle .16}$ \\
\bottomrule
\end{tabular}
\end{table}

\subsection{Ablation Studies}
We conduct ablation studies to evaluate key design choices in the proposed invariant moment architecture, including invariant basis selection (ours vs.\ Flusser), handling of degenerate reference moments (Appendix~\ref{appendix:flexible}), radial function parameterizations (Appendix~\ref{appendix:radial}), and the addition of an equivariant (type-1) branch (Appendix~\ref{appendix:equivariant}). We also report a speed comparison and analyze sample efficiency on the NFFA benchmark, highlighting performance in the low-data regime.

\paragraph{Sample Efficiency.}
We evaluate sample efficiency by training models with various fractions of data from our NFFA benchmark. As shown in Fig.~\ref{fig:data-fraction}, our model consistently outperforms a ResNet18 baseline trained with standard rotation and flip augmentations, particularly in the low-data regime. The performance gap is most pronounced at smaller data fractions (e.g., 6.2\% and 12.5\%), where our approach achieves higher accuracy. As the amount of data increases, the gap narrows.

Notably, these results are obtained with a substantially smaller model: our network has approximately 300k parameters, compared to 11M for ResNet18. This further shows how incorporating group theoretical constructions into ML models can improve sample efficiency.

\paragraph{Speed Comparison.} We compare the forward-pass runtime of our invariant moment layer against $E(2)$-CNNs as a function of the maximum moment type. Results are shown in Fig.~\ref{fig:speed_ablation}. We measure equivalent building blocks: in steerable \acp{CNN} (e.g., \texttt{e2cnn}), this corresponds to two successive equivariant convolutional layers, whereas our construction uses a moment--invariant block of the form $\text{Mom/Inv} \rightarrow 1\times1 \rightarrow \text{Mom/Inv}$. Our method achieves lower runtime; this behavior is consistent with the theoretical complexity. In steerable \acp{CNN}, convolutional layers require interactions between all input and output types, resulting in $\BigO(N_{\text{in}} \cdot N_{\text{out}})$ complexity. In contrast, our construction avoids explicit type-to-type interactions by projecting directly to invariant (type-$0$) features, leading to linear scaling in the number of moments. The empirical results closely follow this trend, with speedups reaching up to $14\times$ for higher moment orders. A minimal implementation of the invariant construction is provided in Appendix~\ref{appendix:implementation}. For example, on the NFFA benchmark with maximum type $k=3$, the full model requires 227 seconds per epoch on a single A100 GPU using \texttt{e2cnn}, compared to 98 seconds for our method.

\paragraph{Invariant Construction.}
We analyze the effect of different invariant constructions. As shown in Table~\ref{tab:invariant_repr}, representations based on moment magnitudes combined with relative phase interactions (Eq.~\eqref{eq:our-basis}) consistently outperform those based on real and imaginary components of Flusser basis~\eqref{eq:flusser_basis}. In practice, we found that the Flusser-based representation requires additional normalization (e.g., a $\sigma$ function applied to the magnitudes) to prevent numerical instabilities due to large magnitude values.

\begin{figure}
\centering
\begin{minipage}{0.49\textwidth}
  \centering
  \captionof{table}{Comparison of invariant representations. Using magnitude and phase-interaction components yields consistently better performance than using real and imaginary parts directly.}
  \label{tab:invariant_repr}
  \begin{tabular}{lccc}
  \toprule
  Dataset 
    & Flusser 
    & Ours \\
  \midrule
  MnistRotTest 
    & $99.27 \pm .05$
    & $\mathbf{99.42 \pm .07}$ \\

  RotMnist 
    & $98.77 \pm .03$ 
    & $\mathbf{99.08 \pm .04}$ \\

  RESISC45 
    & $83.33 \pm .21$ 
    & $\mathbf{85.77 \pm .47}$ \\

  EuroSAT 
    & $93.77 \pm .36$ 
    & $\mathbf{95.29 \pm .40}$ \\
  
  \bottomrule
\end{tabular}
\end{minipage}
\hfill
\begin{minipage}{0.47\textwidth}
  \centering  \includegraphics[width=0.95\linewidth]{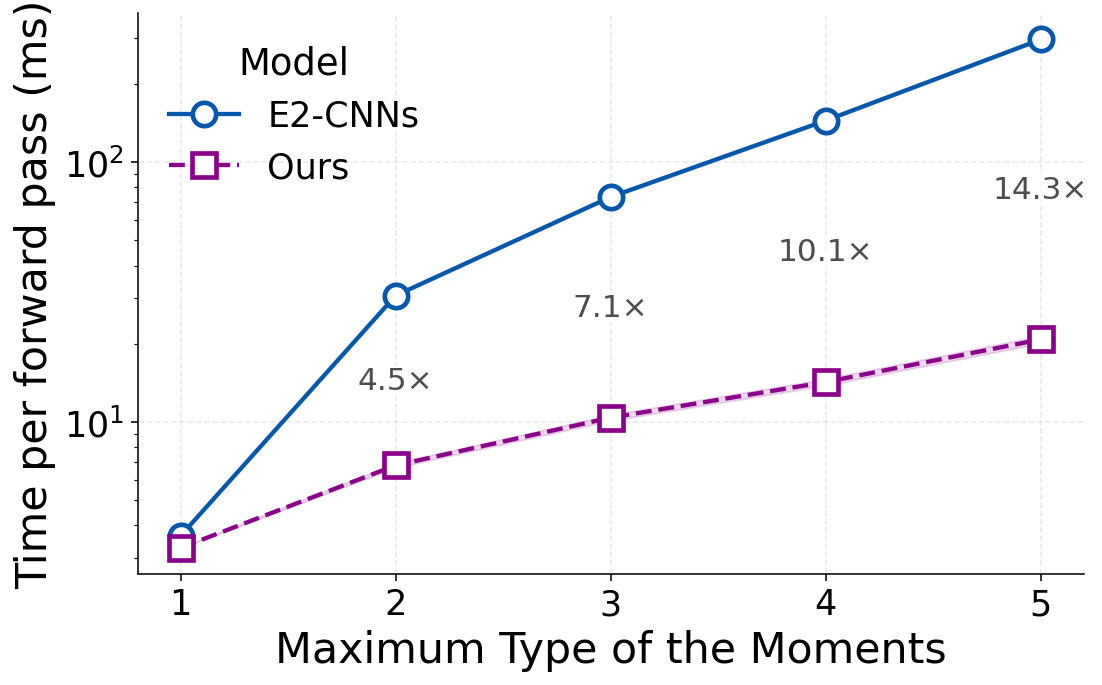}
  \caption{Layer Forward Pass Time}
  \label{fig:speed_ablation}
\end{minipage}
\end{figure}

\section{Conclusion and Limitations}
We introduced an invariant architecture for constructing efficient $\mathrm{SE}(2)$-invariant networks. To connect classical invariant theory with modern deep learning, we derived a minimal and discriminative invariant map that preserves all information up to the $\se$ group action while avoiding redundant features. We translated this construction into a plug-and-play \ac{CNN} module that projects moments directly into invariant type-$0$ features, enabling seamless integration with standard \ac{CNN} layers such as pointwise nonlinearities, normalization layers, and $1\times1$ convolutions.

Empirically, the proposed models achieve competitive or improved accuracy across a range of rotation-symmetric benchmarks while using substantially fewer parameters and with lower runtime than steerable equivariant baselines. These gains are particularly evident in settings involving low data and scientific imaging, where symmetry priors improve sample efficiency. Overall, the results indicate that invariant maps can provide an effective and computationally efficient alternative to fully equivariant features. Notably, performance remains strong despite enforcing local invariance, indicating that relative orientation can still be inferred from invariant features computed over the continuous spatial domain (Appendix~\ref{appendix:equivariant}).

Despite these advantages, the approach is most beneficial for tasks with strong rotational symmetry, where orientation is a nuisance variable. In high-data regimes, similar robustness can often be achieved through data augmentation and increased model capacity. As the theory assumes continuous transformations, invariance is only approximate in discretized implementations. Finally, our invariant map relies on non-zero moment responses; features may become unreliable when moment magnitudes are very small. In Appendix~\ref{appendix:flexible}, we discuss solutions to this issue and show that such cases are rare in practice and can be mitigated by computations over the continuous spatial domain.

Future work may explore adaptive selection of moment types, extensions to broader transformation groups, and applications that combine invariant and equivariant representations. In particular, extending the approach to three-dimensional settings, for example by projecting 3D features through 2D as explored in prior work~\citep{passaro2023}, is a promising direction.

\section{Related Work}
\paragraph{Moment Invariants and Classical Pattern Recognition.}
The construction of invariant features has a long history in pattern recognition. Hu introduced one of the first sets of rotation-invariant moments, demonstrating their effectiveness for shape recognition~\citep{hu1962}. Complementary to moment-based invariants, steerable filters provided an early framework for handling rotations by expressing filter by linear combinations of a fixed set of basis filters~\citep{freeman1991}. This idea introduced a classical form of rotation-equivariant feature extraction and later became an important concept to steerable convolutional architectures. Many early hand-designed invariant sets were redundant or incomplete. Flusser and collaborators formalized the notion of invariant feature systems that are both functionally independent and complete, meaning they can generate any invariant~\citep{flusser2000,flusser2002}. Wood already pointed out that invariant representations should not only remove nuisance redundancy but also retain discriminative information. Subsequent work further explored completeness and independence~\citep{flusser2009}, as well as extensions to more general settings, including tensor fields and flexible invariant bases that avoid degenerate cases~\citep{bujack2017,bujack2022}.

\paragraph{Equivariant Convolutional Networks.}
A large body of work extends \acp{CNN} to encode symmetries beyond translations by enforcing equivariance to transformation groups. Group equivariant \acp{CNN} generalize convolutions to discrete and compact groups~\citep{cohen2016}. Subsequent work extended these ideas to continuous groups, including planar rotations and Euclidean groups, using steerable filters and tools from representation theory. For instance, harmonic networks achieve rotation equivariance by constraining filters to a circular harmonics basis~\citep{worrall2017}. More general frameworks such as steerable CNNs formalize this approach by constructing feature spaces that transform according to irreducible representations of the underlying group, often referred to as steerable fields~\citep{weiler2018,e2cnn,cesa2022,weiler2018b}. Libraries such as \texttt{e2cnn}~\citep{e2cnn} and \texttt{escnn}~\citep{cesa2022} provide practical implementations of these methods. These models typically introduce additional computational and architectural complexity compared to standard \acp{CNN}, and are most commonly applied in 2D settings~\citep{bekkers2018,han2021, pandya2023, mondal2023, santellani2023}.

Beyond $\se$ equivariant CNNs, related work has also explored equivariance in higher-dimensional domains and non-convolutional architectures, including 3D steerable CNNs, tensor field networks, $\mathrm{SE}(3)$-equivariant transformers, and equivariant graph/attention models for point clouds, molecular systems, and physical simulation~\citep{weiler2018,thomas2018,kondor2018,fuchs2020,equiformerv1,equiformerv2,brandstetter2022,musaelian2023, bekkers2024, Allen2026}.

\paragraph{Invariant Features in Deep Learning.}
An alternative to equivariance is to construct invariant representations directly. Early methods achieved invariance by pooling over transformed inputs, as in TI-Pooling~\citep{laptev2016}, by learning canonicalizing transformations, as in scattering transforms~\citep{mallat2012}. Other architectures introduced orientation-aware or rotation-equivariant feature extraction within CNNs, including Oriented Response Networks~\citep{zhou2017} and Rotation Equivariant Vector Field Networks~\citep{marcos2017}.  Recent work has revisited invariance within standard \acp{CNN} through analytically defined Bessel kernels~\citep{delchevalerie2021} and rotation-invariant coordinate convolutions~\citep{mo2024}. However, these methods typically rely on pooling, weight sharing, learned canonicalization, or specific filter parameterizations, and do not explicitly construct complete and non-redundant invariant feature spaces in the classical sense.

\paragraph{Acknowledgemetns}
We thank Kei Davis for feedback on the manuscript. This work is published under LANL LA-UR-26-23375.
We gratefully acknowledge the support of the U.S. Department of Energy through
the LANL Laboratory Directed Research Development Program under project number
20250145ER. 
\bibliographystyle{unsrtnat}
\bibliography{neurips}


\clearpage
\appendix
\section{Flexible Bases}
\label{appendix:flexible}
Flexibility was first formalized by Bujack et al.~\cite{bujack2017}, who observed that the Flusser basis (Eq.~\eqref{eq:flusser_basis}) cannot distinguish patterns for which all type-1 moments vanish. This occurs, for example but not exclusively, for $n$-fold symmetric patterns. For example, in the four-particle configuration in Fig.~\ref{fig:flexible}, rotations affect only the phase of \(m_4=\mu(4,0)\), while the reference moment \(m_1=\mu(1,0)\) remains zero. Since the Flusser basis uses \(m_1\) as a fixed phase reference, the resulting invariants become degenerate for such patterns. Similar phenomena arise in three-dimensional settings, including 
atomic structure representations~\cite{pozdnyakov2020, bujack2025flexible} and geometric graph neural 
networks~\cite{joshi2023}, where distinct configurations may become indistinguishable 
under insufficiently expressive invariant/equivariant features.

\todo[]{Change picture if we have time. It might be lost because of the mac. }
\begin{figure}[h]
  \centering
  \includegraphics[width=\linewidth]{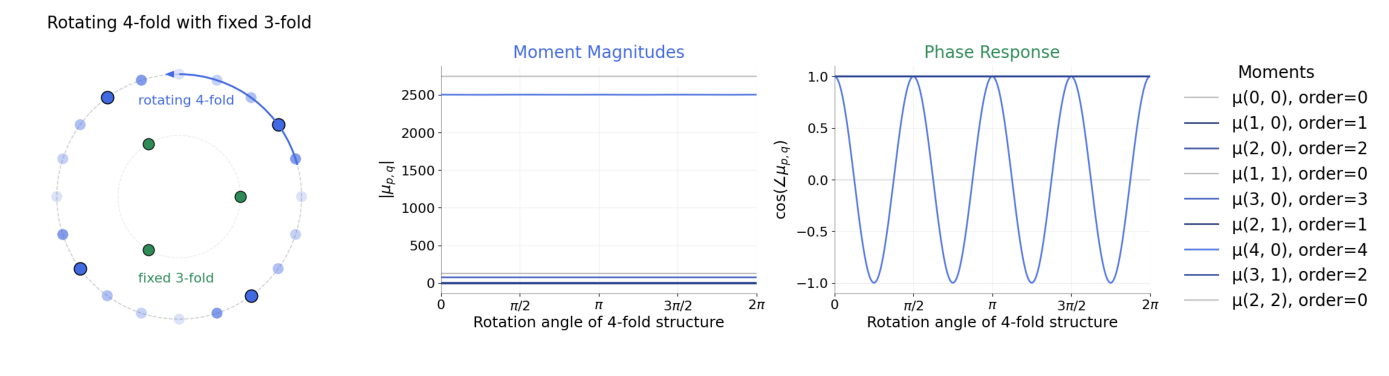}
  \caption{Flexible benchmark and its invariant representation. Left: data generation from rotated three-fold and four-fold structures, with labels determined by their minimal distance. Right: projection into Flusser complex moments, showing the structured embedding induced by invariant features. Inspired by graph neural network benchmarks \cite{joshi2023}.}
  \label{fig:flexible}
\end{figure}

This motivates the flexible basis~\cite{bujack2017}. Rather than fixing \(m_1\) as the reference moment, the flexible construction selects a nonzero moment \(m_k\) as the phase reference. This keeps the representation informative when lower-order moments vanish. If the degeneracy structure is unknown, one can consider each nonzero moment as a possible reference, yielding a larger flexible invariant set at the cost of a quadratic increase in moment combinations. The flexible basis by Bujack et al. \citep{bujack2017} has two shortcomings. The first is that it requires a-priori knowledge about the input pattern. The second is that it requires keeping multiple frames in case of rotationally symmetric patterns. In the next section, we will show that we can overcome both.

\subsection{Synthetic flexibility benchmark}
We test this effect in a pixel-based synthetic classification task. Each sample contains a three-fold and a four-fold structure on circles of different radii. The binary label is determined by their minimal distance, split at the median to obtain a balanced dataset. The point configurations are projected onto a regular grid and smoothed with a Gaussian kernel to reduce interpolation artifacts from discretization and rotation.

We evaluate a simple convolutional model in which a single Mon/Inv-layer filter spans the full spatial resolution of the image under valid convolution, producing a \(1 \times 1\) feature map that is passed to an MLP classifier. With the non-flexible Flusser basis (Eq.~\eqref{eq:flusser_basis}) and non-flexible optimal basis (Eq.~\eqref{eq:our-basis}), the model achieves only chance-level accuracy, approximately \(50\%\). In contrast, applying the flexible construction to the optimal basis increases accuracy to approximately \(99\%\).

Interestingly, adding padding allows the non-flexible bases to recover performance comparable to the flexible bases. Padding introduces boundary context that can break the perfect symmetry seen by the valid convolution, thereby compensating for the degenerate reference moment. Thus, padding acts as an architectural shortcut: it restores empirical performance without resolving the underlying limitation of the non-flexible basis.

\subsection{Implications for Dense Invariant \acp{CNN}}
In our \acp{CNN}, replacing the optimal basis with its flexible variant did not lead to measurable improvements in classification accuracy. To understand this, we measured how often candidate reference moments vanish across channels and spatial locations. Table~\ref{tab:empflex} reports the fraction of near-zero moments for several datasets.

The observed frequencies are extremely small across all datasets and moment orders. This suggests that, in dense pixel or feature-map representations, exact moment degeneracies are rare. Unlike sparse point or graph representations, where highly symmetric configurations can make reference moments vanish exactly, spatially extended convolutional features and image discretization tend to break such degeneracies.

Together with the synthetic benchmark in Fig.~\ref{fig:flexible}, these results suggest that flexibility primarily matters in settings with sparse or highly symmetric inputs, where a chosen reference moment may vanish while alternative moments remain available. In contrast, such degeneracies occur only rarely in the dense feature representations produced by our \acp{CNN} (tested before and after training). This differs from settings such as geometric graph neural networks~\cite{joshi2023}, where sparsity persists throughout the network and indistinguishable configurations arise more frequently.

\begin{table}[h]
\centering
\begin{tabular}{lccccc}
\toprule
\textbf{Dataset} 
& $\mathbf{m_{1,1}}$
& $\mathbf{m_{2,0}}$
& $\mathbf{m_{2,1}}$  
& $\mathbf{m_3}$ 
& $\mathbf{m_4}$ \\
\midrule
Histology 
& $8.7\mathrm{e}{-5}$ 
& $2.7\mathrm{e}{-4}$ 
& $2.1\mathrm{e}{-4}$ 
& $7.1\mathrm{e}{-5}$ 
& $5.1\mathrm{e}{-5}$ \\

MNIST 
& $1.4\mathrm{e}{-6}$ 
& $1.4\mathrm{e}{-6}$ 
& $1.4\mathrm{e}{-6}$ 
& $1.4\mathrm{e}{-6}$ 
& $1.4\mathrm{e}{-6}$ \\

Satellite 
& $1.8\mathrm{e}{-5}$ 
& $2.2\mathrm{e}{-4}$ 
& $4.9\mathrm{e}{-4}$ 
& $2.3\mathrm{e}{-6}$ 
& $1.4\mathrm{e}{-6}$ \\
\bottomrule
\end{tabular}
\caption{Conditional probability that another moment is nonzero given that the reference moment \(m_{1,0}\) is zero, measured across channels and spatial locations. Larger values indicate cases where flexible reference selection could avoid degeneracies caused by a vanishing reference moment. The consistently small probabilities indicate that situations where an alternative reference moment could be used are exceedingly rare in these CNN feature representations.}
\label{tab:empflex}
\end{table}

\section{Optimal Invariant Descriptors}
\label{appedix:proof}
In this appendix, we show how to derive an optimal map of descriptors, by which we mean that they are discriminative and minimal, i.e., the descriptor is able to distinguish any two inputs that differ by something other than a rotation, and that removing any element breaks that property~\cite{wood1996}.

After some notation, we will first introduce the \textbf{Bézout map}, which can discriminate a known input from every other pattern. Then, we describe the \textbf{minimal discriminative map}. We prove that it achieves general discriminability without a-priori knowledge by deriving any Bézout map from its descriptors. We also prove invariance and minimality for both $\oo$ and $\so$.

Note that discriminability refers to a set of admissible input patterns. Generally, we consider patterns only up to a maximal finite order. Flusser's basis can discriminate any pattern that has a non-zero moment of type one. The Bézout map is more powerful. It can discriminate any pattern as long as we know it a-priori. The flexible map is the most powerful one. It can discriminate any patterns without prior knowledge.

\subsection{From Complex to Fourier Moments}
Complex moments are typically used with two indices~\cite{flusser2000}. To make notation more concise and generalize to moments with other radial parts, such as Zernike moments, in this section we split angular from radial parts.

Let \(f:\mathbb{R}^2\to\mathbb{R}\), and write
\(z=x+iy=re^{i\theta}\). Assume the following moments exist. The complex
moments are
\begin{equation}\label{eq:c_pq}
c_{p,q}
=
\int_{\mathbb{R}^2} z^p\bar z^q f(x,y)\,dx\,dy .
\end{equation}
In polar coordinates,
\begin{equation}\label{eq:c_pq_polar}
c_{p,q}
=
\int_0^\infty\int_0^{2\pi}
r^{p+q} e^{i(p-q)\theta} f(r,\theta)\, r\,d\theta\,dr .
\end{equation}
Thus under rotation $f_\alpha(r,\theta) = f(r,\theta-\alpha)$ the moment \(c_{p,q}\) transforms with type
\begin{equation}\label{eq:spin}
k:=p-q,
\end{equation}
and
\begin{equation}\label{eq:rotation_c}
c_{p,q}(f_\alpha)=e^{ik\alpha}c_{p,q}(f).
\end{equation}

Hence we reduce to Fourier moments
\begin{equation}\label{eq:m_k}
m_k:=\int_0^{2\pi} f(\theta)e^{ik\theta}\,d\theta,
\qquad k\in\mathbb{Z}.
\end{equation}
For real-valued \(f\),
\begin{equation}\label{eq:real_sym}
m_{-k}=\overline{m_k}.
\end{equation}
Under rotation \(f_\alpha(\theta)=f(\theta-\alpha)\), the moments transform via
\begin{equation}\label{rot}
m_k(f_\alpha)
=
\int_0^{2\pi} f(\theta-\alpha)e^{ik\theta}\,d\theta
\;\overset{u=\theta-\alpha}{=}\;
\int_0^{2\pi} f(u)e^{ik(u+\alpha)}\,du
=
e^{ik\alpha} m_k(f).
\end{equation}
Under reflection \(\tilde f(\theta)=f(-\theta)\), the moments transform via
\begin{equation}\label{refl}
m_k(\tilde f)
=
\int_0^{2\pi} f(-\theta)e^{ik\theta}\,d\theta
\;\overset{u=-\theta}{=}\;
\int_0^{2\pi} f(u)e^{-iku}\,du
=
m_{-k}(f)
\;\overset{\eqref{eq:real_sym}}{=}\;
\overline{m_k(f)}.
\end{equation}

\subsection{Bézout Map}
In this section, we show that a basis of descriptors can be chosen if the input is known a-priori. For a given set of input moments, we derive a map that discriminates it from any other set of moments using Bézout's identity~\cite{hardy1979introduction}.

Let
\begin{equation}\label{eq:M_def}
M=\{m_k \mid k\in\mathbb{N},\; k\le N_k\}
\end{equation}
be the set of Fourier moments of a real-valued function
\begin{equation}\label{eq:f_def}
f:S^1\to\mathbb{R}
\end{equation}
up to order \(N_k\in\mathbb{N}\).
In case all moments are zero, we define all descriptors to be zero as well. From now on, we assume at least one moment is non-zero.

Let
\begin{equation}\label{eq:S_def}
S:=\{\, k \mid m_k\neq 0,\; k\in\mathbb{N},\; k\le N_k \,\}
\end{equation}
be the set of indices with non-zero moments, and let
\begin{equation}\label{eq:g_def}
d=\gcd(S)
\end{equation}
be their greatest common divisor. Then all non-zero moment indices are multiples of $d$,
\begin{equation}\label{eq:g_divides_k}
\forall\, m_k\neq 0:\quad d\mid k,
\quad\text{i.e.}\quad \frac{k}{d}\in\mathbb{N}.
\end{equation}

If zero moments are present, we set their Bézout coefficients to zero, i.e., \(a_k:=0\). By Bézout's identity there exist integers \(a_k\in\mathbb{Z}\) such that
\begin{equation}\label{eq:bezout}
\sum_{k=0}^{N_k} a_k\, k = d.
\end{equation}

This yields a rotational normalizer of type \(d\),
\begin{equation}\label{eq:mprime}
m' := \prod_{k=0}^{N_k} m_k^{a_k},
\qquad 
\mathrm{type}(m')=\sum_{k=0}^{N_k} a_k k=d.
\end{equation}

To simplify the inverse problem, we normalize \(m'\) to unit magnitude:
\begin{equation}\label{eq:m_norm}
m:=\frac{m'}{\sqrt{m'\overline{m'}}},
\end{equation}
where the square root denotes the positive one.

The Bézout map is then defined as
\begin{equation}\label{eq:Psi_def}
\Psi := \{\, \psi_k \mid k = 0,\dots, N_k \,\}.
\end{equation}

For each \(k\), we define
\begin{equation}\label{eq:psi_general}
\psi_k
=
m_k^{\frac{d}{\gcd(d,k)}}
\, m^{\frac{-k}{\gcd(d,k)}}.
\end{equation}

Since \(d\mid k\) for all non-zero moments, we have \(\gcd(d,k)=d\), and thus
\begin{equation}\label{eq:psi_simplified}
\psi_k
=
m_k^{\frac{d}{d}}
\, m^{-\frac{k}{d}}
=
m_k\, m^{-\frac{k}{d}}.
\end{equation}

This expression is well-defined since \(\frac{k}{d}\in\mathbb{N}\). 

\subsubsection{Invariance}\label{bezoutInv}
\begin{theorem}
The Bézout map is invariant to rotations.
\end{theorem}

\begin{proof}
We know from~\eqref{rot} that under rotations $\theta\mapsto\theta-\alpha$
\begin{equation}
m_k \mapsto e^{ik\alpha}m_k,
\qquad
m \mapsto e^{ig\alpha}m.
\end{equation}
Hence
\begin{equation}
\psi_k
=
m_k m^{-k/d}
\mapsto
e^{ik\alpha}m_k
\left(e^{id\alpha}m\right)^{-k/d}
=
e^{ik\alpha}m_k e^{-ik\alpha}m^{-k/d}
=
\psi_k.
\end{equation}
\end{proof}

The Bézout map is not invariant to reflections. It follows from~\eqref{refl} that the imaginary part of each descriptor changes sign upon reflection while its real part remains unchanged. Therefore, the Bézout map can be easily adapted to be invariant to rotations and reflections by replacing each complex descriptor by its magnitude and real part
\begin{equation}
\Psi_O := \{\, |\psi_k|, \Re(\psi_k) \mid k = 0,\dots, N_k \,\}.
\end{equation}
A more flexible/robust choice, though, is to use the magnitude of the moments and the real part. If a pattern with a different support is used than the one for which the Bézout map was created, its normalizer can vanish ($m=0$), which causes every element of the Bézout map to vanish, except the trivial ones with $k=0$, because they are all products of $m$. Instead, using 
\begin{equation}
\Psi^{\prime}_O := \{\, |m_k|, \Re(\psi_k) \mid k = 0,\dots, N_k \,\}
\end{equation}
keeps all the magnitudes usable, maintaining half the information.

\subsubsection{Discriminability}
\begin{theorem}
The Bézout map is discriminative.
\end{theorem}

\begin{proof}
The full Fourier transform is injective on the function class considered. Since we only use moments up to order $N_k$, it suffices here to prove discriminability of the truncated moment vectors up to the rotation action.

We solve the inverse problem analogously to Flusser~\cite{flusser2002} by working in the canonical gauge, i.e., in the chosen standard position where \(m'\in\mathbb R^+\). Then, we immediately know that \(m=1\) and the reconstruction of the moments becomes trivial
\begin{equation}
m_k=\psi_k m^{k/d}=\psi_k.
\end{equation}
\end{proof}

\subsubsection{Minimality}
The Bézout Map is not in general independent or minimal. Assume that there exists an index
$l$ such that
$a_l = 1$, then we can replace the complex $\psi_l$ with only its norm $|\psi_l|$. 

We construct $\psi_l$ in the canonical gauge. There the normalized B\'ezout normalizer satisfies $m'=|m'|\in\mathbb{R}^+$, which leads to $m=1$ and
\begin{equation}\label{canonical}
\psi_k \overset{\eqref{eq:psi_simplified}}= m_km^{\frac{-k}d} = m_k.
\end{equation}
Using \eqref{eq:mprime} and $|\psi_l|$, we can first compute
\begin{equation}
m'=|m'|=\prod_{k=0}^{N_k} |m_k^{a_k}|.
\end{equation}
Then since \(a_l=1\), we can isolate \(m_l\)
\begin{equation}
\psi_l
\overset{\eqref{canonical}}=
m_l
\overset{\eqref{eq:mprime}}=
m'\prod_{\substack{k=0\\k\neq l}}^{N_k} m_k^{-a_k}
\overset{\eqref{canonical}}=
m'
\prod_{\substack{k=0\\k\neq l}}^{N_k}
\psi_k^{-a_k}.
\end{equation}
An analogous result can be derived for $a_l=-1$.
In both cases, the map 
\begin{equation}
\Psi_{min} := \{\, \psi_0,...,\psi_{l-1},|\psi_l|,\psi_{l+1},...,\psi_{N_k}\}
\end{equation}
is minimally discriminative.


\subsubsection{Relation to Flusser's Basis}

If \(m_1\neq 0\), we have $d=1$ and can directly choose \(a_1=1,a_2=0,...,a_{N_k}=0\), which leads to \(m'=m_1\),
\begin{equation}
m=\frac{m_1}{|m_1|}=\left(\frac{\overline{m_1}}{|m_1|}\right)^{-1}
\end{equation}
and
\begin{equation}\label{eq:flusser}
\psi_k
\overset{\eqref{eq:psi_simplified}}=
m_km^{\frac{-k}d}
\overset{\eqref{eq:mprime}}=
m_k\left(\frac{\overline{m_1}}{|m_1|}\right)^{k}.
\end{equation}
This is exactly the Flusser basis~\cite{flusser2000} up to normalization.

Analogously to Sec.~\ref{bezoutInv}, we can make Flusser's basis 
\begin{equation}
\Phi_{F} = \{\, |m_0|, |m_1|,\psi_k \mid k = 2,\dots,o \,\}
\end{equation}
invariant to rotations and reflections by 
\begin{equation}
\{\, |m_0|, |m_1|, |\psi_k|, \Re(\psi_k) \mid k = 2,\dots, N_k \,\}
\end{equation}
and more flexible and robust via
\begin{equation}
\Psi_{\mathcal{O}}=\{\, |m_0|,|m_1|, |m_k|, \Re(\psi_k) \mid k = 2,\dots, N_k \,\}.
\end{equation}
For a pattern with a $m_1=0$, all nontrivial elements of Flusser's basis vanish because they contain powers of $m_1$. We use the flexible one, which maintains all magnitude information. 

Note that this variant is not fully discriminative. However, among descriptors that retain only one phase descriptor per order, it is optimal in the sense that it preserves all available magnitude information while keeping the descriptor cardinality fixed. In the next part, we will see how full discriminability can be achieved by adding more descriptors.

```latex
\subsection{The Minimal Discriminative Map}
The previous subsection assumed that we know the input pattern. In a general machine learning model, this assumption does not hold. In this subsection, we derive a minimal discriminative set for arbitrary, a-priori unknown input patterns. We will show that there is no basis, because all discriminative descriptor sets are algebraically dependent. Still, there is an optimal map of descriptors, as it is the smallest that can discriminate any two inputs.

We construct the minimal flexible set in the following way. 
Let
\begin{equation}\label{eq:M_fourier}
M=\{m_k\mid k\in\mathbb{Z},\; |k|\le N_k\}
\end{equation}
be the set of Fourier moments of an arbitrary real-valued function
\begin{equation}\label{eq:f_real_fourier}
f:S^1\to\mathbb{R}
\end{equation}
up to a given order \(N_k\in\mathbb{N}\).

The minimal flexible set consists of magnitudes and all pairwise combinations
\begin{equation}\label{eq:B_fourier}
\Phi:=\{\,|m_k|\mid k=0,\dots, N_k\,\}
\;\cup\;
\{\phi(k,j)\mid 1 \le k< j \le N_k\},
\end{equation}
of the form
\begin{equation}\label{eq:phi_fourier_def}
\phi(k,j):=
m_k^{\frac{j}{d_{kj}}}m_{j}^{-\frac{k}{d_{kj}}}
\end{equation}
with 
\begin{equation}
d_{kj}:=\gcd(k,j).
\end{equation}
If $m_{j}=0$, we define $\phi(k,j):=0$.
We need to prove invariance, discriminability, and minimality.

\subsubsection{Invariance}
\begin{theorem}
The minimal discriminative map \(\Phi\) is invariant under rotations.
\end{theorem}

\begin{proof}
We show that every element of \(\Phi\) is invariant under planar rotations $f_\alpha(\theta):=f(\theta-\alpha)$.
First, the magnitudes are invariant because of
\begin{equation}
|m_k(f_\alpha)|\overset{\eqref{rot}}=|e^{ik\alpha}m_k(f)|=|m_k(f)|.
\end{equation}

Then, for any two moments $m_k,m_j\in M$ with $k,j\ge 1$, the mixed terms satisfy
\begin{align}
\phi_\alpha(k,j)
&\overset{\eqref{eq:phi_fourier_def}}=
m_k(f_\alpha)^{j/d_{kj}}\;
m_j(f_\alpha)^{-k/d_{kj}} \\
&\overset{\eqref{rot}}=
\Big(e^{ik\alpha}m_k(f)\Big)^{j/d_{kj}}\;
\Big(e^{ij\alpha}m_j(f)\Big)^{-k/d_{kj}} \\
&=
e^{ikj\alpha/d_{kj}}\;
e^{-ijk\alpha/d_{kj}}\;
m_k(f)^{j/d_{kj}}m_j(f)^{-k/d_{kj}} \\
&\overset{\eqref{eq:phi_fourier_def}}=
\phi(k,j).
\end{align}
Thus each $\phi(k,j)$ is rotation invariant, and consequently the entire map $\Phi$ is invariant under planar rotations.
\end{proof}

$\Phi$ is not invariant to reflections. It follows from~\eqref{refl} that the imaginary part of each mixed term changes sign upon reflection while its real part remains unchanged. Therefore, $\Phi$ can be easily adapted to be invariant to rotations and reflections by replacing each complex descriptor by its magnitude and real part
\begin{equation}
\Phi_O:=\{\,|m_k|\mid k=0,\dots, N_k\,\}
\;\cup\;
\{|\phi(k,j)|, \Re(\phi(k,j)) \mid 1 \le k< j\le N_k\}.
\end{equation}

\subsubsection{Discriminability}
\begin{theorem}
The minimal discriminative map \(\Phi\) is discriminative.
\end{theorem}

\begin{proof}
We show that any B\'ezout invariant $\psi_k$ can be constructed from the flexible set $\Phi$. Then, discriminability of $\Phi$ follows from the discriminability of the B\'ezout map.

First, we derive all $\phi(k,j)$ with $k>j$, which are not part of $\Phi$, from
\begin{equation}
\phi(k,j)\overset{\eqref{eq:phi_fourier_def}}=\phi(j,k)^{-1}
\end{equation}
if $\phi(j,k)\neq0$ and zero otherwise.

Then since all magnitudes $|m_k|$ are part of $\Phi$, we know which moments are non-zero and part of the support \(S\)~\eqref{eq:S_def}, so we know the type \(d\)~\eqref{eq:g_def} and B\'ezout coefficients $a_j$~\eqref{eq:bezout}.
For \(k\in S\), we obtain
\begin{align}
\prod_{j\in S}
\phi(k,j)^{a_j d_{kj}/d}
&\overset{\eqref{eq:phi_fourier_def}}=
\prod_{j\in S}
\left(
m_k^{j/d_{kj}}m_j^{-k/d_{kj}}
\right)^{a_j d_{kj}/d} \\
&=
m_k^{\sum_{j\in S} a_j j/d}
\prod_{j\in S} m_j^{-a_j k/d} \\
&\overset{\eqref{eq:bezout}}=
m_k\,\prod_{j\in S} m_j^{-a_j k/d}\\
&\overset{\eqref{eq:mprime}}=
m_k\,(m')^{-k/d}.
\end{align}
Since all magnitudes $|m_j|$ are part of $\Phi$, we can derive
\begin{equation}
|m'|
\overset{\eqref{eq:mprime}}=
\left|\prod_{j\in S} m_j^{a_j}\right|
=
\prod_{j\in S}|m_j|^{a_j}.
\end{equation}
Together, we obtain
\begin{equation}
\psi_k
\overset{\eqref{eq:psi_simplified}}=
m_k m^{-k/g}
\overset{\eqref{eq:m_norm}}=
m_k \left(\frac{m'}{|m'|}\right)^{-k/g}
=
\left(
\prod_{j\in S}
\phi(k,j)^{a_j d_{kj}/d}
\right)
|m'|^{k/d}.
\end{equation}
Thus every B\'ezout invariant $\psi_k$ can be constructed from $\Phi$. For $k\notin S$, both $m_k$ and $\psi_k$ vanish. Hence $\Phi$ is discriminative.
\end{proof}

\subsubsection{Minimality}
\begin{theorem}
The minimal discriminative map \(\Phi\) is minimal.
\end{theorem}

\begin{proof}
We will show that removing any element causes two sets of moments that differ by something other than a rotation to be indistinguishable. 

Removing a magnitude $|m_k|$ leads to the moment sets $m^{(1)}_{k}=1$, others zero becoming indistinguishable from $m^{(2)}_{k}=2$, others zero because except for $|m_k|$, all elements of $\Phi$ are zero for both.

Fix an admissible pair $(k,j)$ and consider $\Phi' := \Phi\setminus\{\phi(k,j)\}$.
We exhibit the two \emph{different} moment sets. For set 1
\begin{equation}
m^{(1)}_{k}=1,\quad m^{(1)}_{j}=1,
\end{equation}
others zero, and for set 2
\begin{equation}
m^{(2)}_{k}=e^{i\pi d_{kj}/j},\quad m^{(2)}_{j}=1,
\end{equation}
others zero.

Clearly, all magnitudes agree. 
For any remaining descriptor $\phi(p,q)\in\Phi'$, if $\{p,q\}\not\subseteq\{k,j\}$, then at least one involved moment is zero in both moment sets, hence both descriptors are zero. 
Thus all descriptors in $\Phi'$ coincide, while $\phi(k,j)$ separates the two moment sets.

We show by contradiction that the moments are not related by rotation $e^{i\alpha}$. Assume a rotation exists. Then for some \(\alpha\),
\begin{equation}
m_k^{(2)}=e^{ik\alpha}m_k^{(1)},
\qquad
m_j^{(2)}=e^{ij\alpha}m_j^{(1)}.
\end{equation}

From the \(j\)-moment,
\begin{equation}
1=e^{ij\alpha},
\end{equation}
hence
\begin{equation}
j\alpha=2\pi r,
\qquad r\in\mathbb{Z},
\end{equation}
and therefore
\begin{equation}
\alpha=\frac{2\pi r}{j}.
\end{equation}

Using the \(k\)-moment, we require
\begin{equation}
e^{i\pi d_{kj}/j}
=
e^{ik\alpha}.
\end{equation}
Substituting \(\alpha\) gives
\begin{equation}
e^{i\pi d_{kj}/j}
=
e^{2\pi i k r/j}.
\end{equation}

Equality of exponentials implies
\begin{equation}
\frac{\pi d_{kj}}{j}
=
\frac{2\pi kr}{j}
+
2\pi q,
\qquad q\in\mathbb{Z}.
\end{equation}
Dividing by \(\pi\) and multiplying by \(j\) yields
\begin{equation}
d_{kj}
=
2kr+2qj,
\end{equation}
i.e.,
\begin{equation}
2kr\equiv d_{kj}\pmod{2j}.
\end{equation}

Write
\begin{equation}
k=ad_{kj},
\qquad
j=bd_{kj},
\qquad
\gcd(a,b)=1.
\end{equation}
Then
\begin{equation}
2ag_{kj}r\equiv d_{kj}\pmod{2bd_{kj}}.
\end{equation}
Dividing by \(d_{kj}\) gives
\begin{equation}
2ar\equiv 1\pmod{2b}.
\end{equation}
This is impossible because the left-hand side is even while the right-hand side is odd. Hence no such rotation exists.

Thus every descriptor is necessary, and \(\Phi\) is minimally discriminative.
\end{proof}

\subsubsection{Dependence}
The set \(\Phi\) is minimal with respect to discriminability, but not independent. Indeed,
\begin{equation}
|\phi(k,j)|^2
=
\frac{|m_k|^{2j/d_{kj}}}{|m_j|^{2k/d_{kj}}}.
\end{equation}

\section{Equivariant Structure}
\label{appendix:equivariant}
Recent work on equivariant neural networks has emphasized the importance of preserving relative orientation information across features~\citep{weiler2023, romero2020}. In particular, Romero et al.\ argue that fully invariant local representations may discard geometric relationships required for object composition~\citep{romero2020}. Fig.~\ref{fig:equivariance_motivation} illustrates this phenomenon: although the local facial components are identical, disrupting their relative orientations and spatial arrangement destroys the perception of a coherent face.

\begin{figure}[h]
    \centering
    \includegraphics[width=0.95\linewidth]{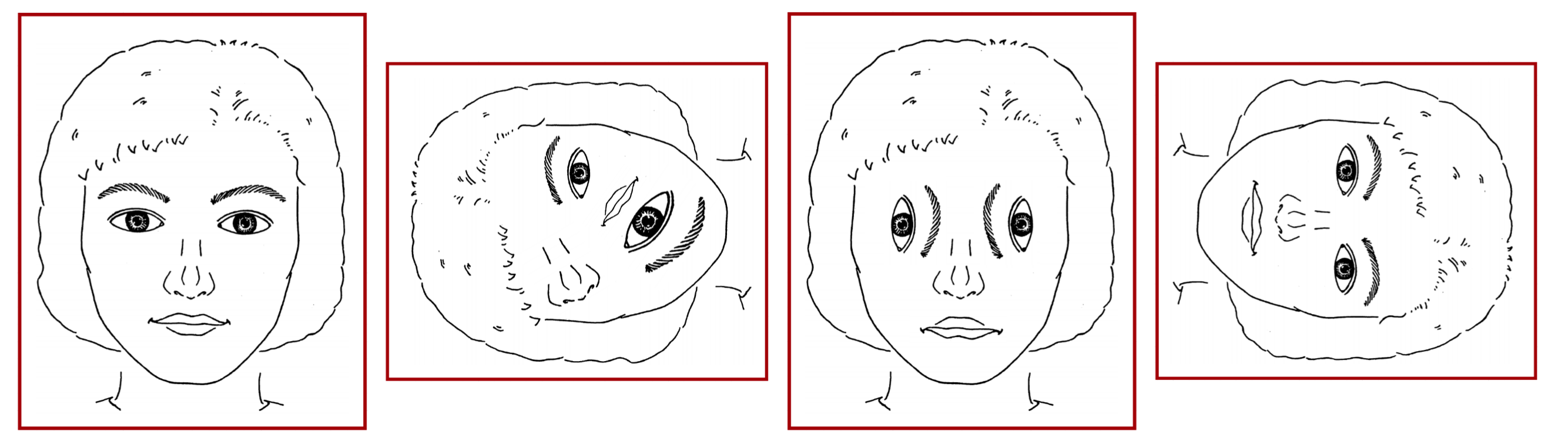}
    \caption{Relative pose relationships are essential for object composition.
All panels contain the same local facial components, but only the outer panels
preserve a coherent global face configuration. The inner panels disrupt the
relative positions, orientations, and scales of the parts, illustrating that
independent invariance of local elements can discard information needed to
recognize the object as a face. Taken from Romero et al.~\citep[Fig.~1]{romero2020},
based on Schwarzer~\citep{schwarzer2000}.}
    \label{fig:equivariance_motivation}
\end{figure}

Motivated by these observations, we investigated whether explicitly retaining equivariant information alongside invariant features improves performance in our setting. In addition to the invariant basis introduced in Eq.~\eqref{eq:our-basis}, we considered variants that preserve a subset of equivariant moment information, such as:
(i) retaining the reference moment $m_1$, which encodes local orientation information, (ii) projecting moment features jointly into type-$0$ (invariant) and type-$1$ (equivariant) representations.\footnote{Using e2cnn layers for dealing with type-1 features.}

Empirically, we did not observe consistent improvements from including explicit equivariant branches. Across the evaluated benchmarks, models based solely on invariant features achieved comparable or better performance than architectures augmented with equivariant components. This behavior is consistent with the quantitative results in the main text, where the proposed invariant models perform competitively with fully equivariant baselines despite operating entirely in the invariant feature space.

One possible explanation is that, although each local feature is individually invariant, relative orientation information can still be inferred implicitly from spatial arrangements across the continuous feature maps. Thus, global geometric structure may remain accessible without explicitly propagating equivariant feature fields through the network.

To further examine this question, we evaluated a variant of the architecture that augments the invariant branch with an additional type-$1$ equivariant branch. Results on RotMNIST are summarized in Table~\ref{tab:ablation_equivariance}. The proposed invariant architecture achieves the highest accuracy while maintaining essentially the same parameter count. Adding an explicit equivariant branch did not improve performance and produced results comparable to the equivariant baseline.

\begin{table} 
    \centering
    \caption{Comparison between invariant and equivariant variants on RotMNIST. Adding an explicit equivariant branch does not improve performance over the invariant architecture.}
    \label{tab:ablation_equivariance}
      \begin{tabular}{lcc}
    \toprule
    Architecture & Accuracy & \# Params \\
    \midrule
    E2CNN              & $98.82 \pm .01$            & 362 K \\
    Mom/Inv \acp{CNN}   & $\mathbf{99.00 \pm .03}$ & 362 K \\
    Equivariant Branch & $98.82 \pm .06$            & 363 K \\
    \bottomrule
  \end{tabular}
\end{table}

\section{Radial Basis Design}
\label{appendix:radial}

The proposed moment layer depends on the choice of radial basis functions \(R_{k,l}(r)\) used in the separable kernel construction of Eq.~\eqref{def:polynomial}. We therefore investigate the sensitivity of the architecture to different radial discretizations and parameterizations.

We first study the effect of the number of radial rings used to discretize the radial component. Results are reported in Table~\ref{tab:ring_count}. Performance remains stable across a broad range of ring counts, with a slight optimum observed around four rings. Next, we compare three radial basis parameterizations: B-splines, monomials used in classical pattern recognition~\citep{flusser2009}, and orthogonal Legendre polynomials. The results are summarized in Table~\ref{tab:kernel_variants}.

\begin{table}
  \caption{Comparison of different radial basis parameterizations. Performance remains consistent across variants and simplified radial constructions. }
  \label{tab:kernel_variants}
  \centering
  \small
  \begin{tabular}{lccc}
  \toprule
  Variant 
    & Test Accuracy 
    & Val Accuracy 
    \\
  \midrule
  Monomials
    & $98.82$ 
    & $99.25$ 
    \\

  Legendre polynomials 
    & $98.80$ 
    & 98.97 
    \\

  B-splines 
    & $\mathbf{99.07}$ 
    & 99.74 
    \\
    
  \bottomrule
\end{tabular}
\end{table}

\todo[]{Nice to have: This can be graph}
\begin{table}
  \caption{Effect of the number of radial rings. Performance is stable across a range of values, with a slight optimum at 4 rings.}
  \label{tab:ring_count}
  \centering
  \begin{tabular}{lccc}
  \toprule
  Ring Count 
    & Test Accuracy 
    & Train Accuracy 
    & Val Accuracy \\
  \midrule
  4  
    & $\mathbf{99.15}$ 
    & 99.93 
    & 99.87 \\

  7  
    & $99.15$ 
    & 99.98 
    & $\mathbf{99.89}$ \\

  3  
    & $99.13$ 
    & 99.89 
    & 99.80 \\

  6  
    & $99.13$ 
    & 99.96 
    & 99.89 \\

  5  
    & $99.09$ 
    & 99.95 
    & 99.83 \\

  8  
    & $99.02$ 
    & 99.77 
    & 99.22 \\
  \bottomrule
\end{tabular}
\end{table}
\todo[]{Nice to have: Pull the fixed basis results. }

\section{Filter Size Trade-off}
We further analyze the influence of filter size on rotational generalization. Results are shown in Fig.~\ref{fig:filter-size-tradeoff}, where we evaluate validation accuracy on both upright and rotated inputs without data augmentation.

As the filter size increases, performance on upright numerals remains consistently high, while accuracy on rotated samples improves substantially. This suggests that larger receptive fields better capture rotational structure by integrating information over a wider spatial neighborhood. In contrast, smaller filters are more sensitive to local orientation changes and therefore generalize less effectively to unseen rotations.

These results highlight a trade-off between local precision and rotational robustness. While very small filters may better preserve fine-grained local structure, larger filters provide improved stability under rotations with only minimal impact on standard upright performance.

\begin{figure} 
    \centering
    \includegraphics[width=0.6\linewidth]{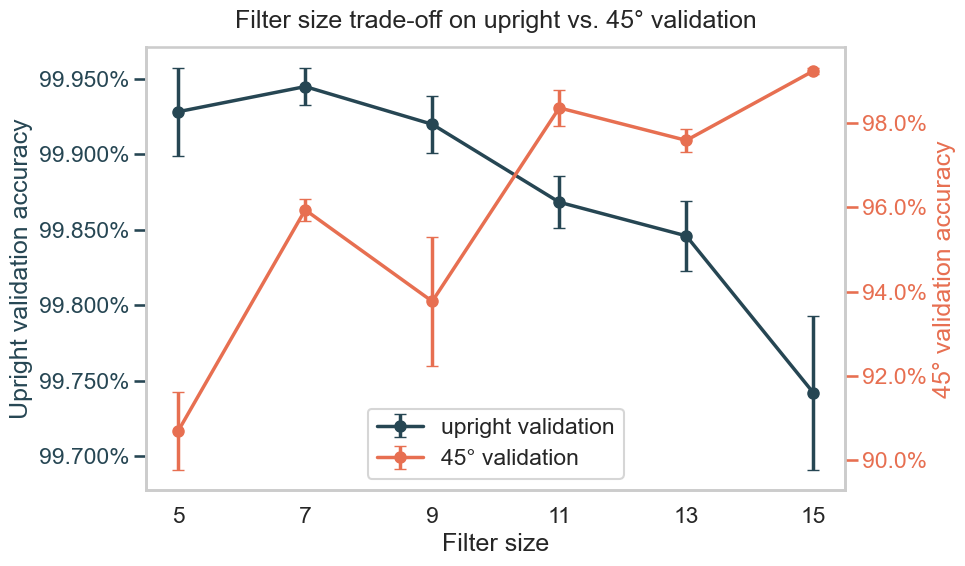}
    \caption{
        Effect of filter size on validation performance without data augmentation.
        The left y-axis shows validation accuracy on upright digits, while the right y-axis shows validation accuracy on digits rotated by $45^\circ$.
        Larger filters preserve very high upright accuracy while substantially improving performance on rotated inputs, illustrating the trade-off between local precision and rotational robustness.
        Error bars indicate one standard deviation across runs.
    }
    \label{fig:filter-size-tradeoff}
\end{figure}

\section{NFFA-Europe}
\label{appendix:nffa}
\begin{figure}
    \centering
    \includegraphics[width=0.80\linewidth]{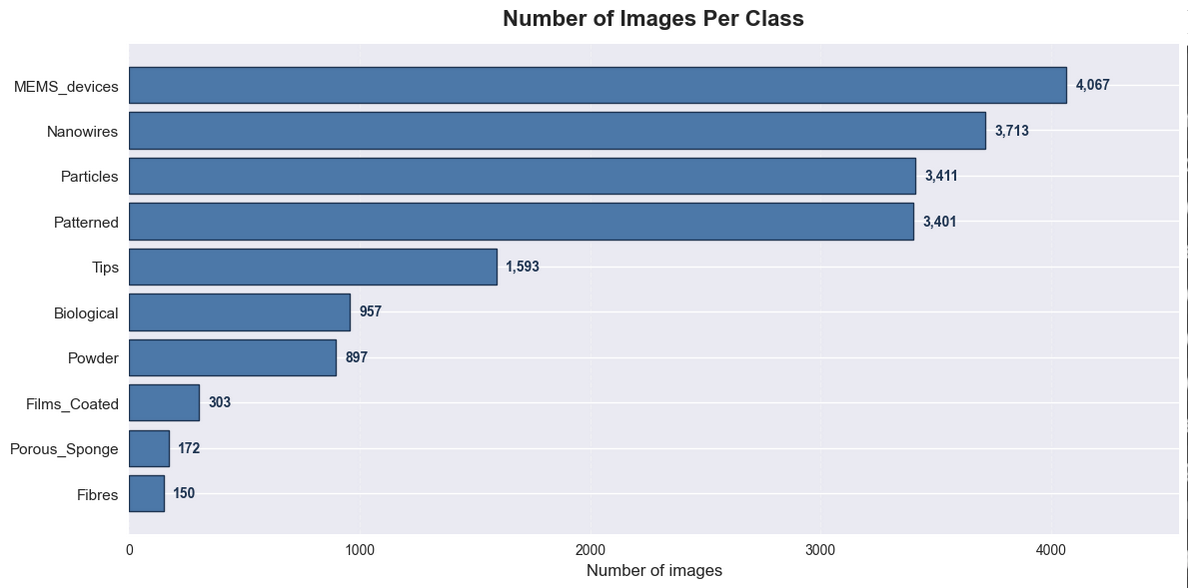}
    \caption{Distribution of SEM images across classes in the NFFA-Europe dataset. The dataset exhibits significant class imbalance, with some categories containing substantially fewer samples than others. This makes the benchmark particularly relevant for evaluating symmetry-aware architectures in low-data regimes.}
    \label{fig:nffa_counts}
\end{figure}
\todo[]{Nice to have some examples}
To evaluate the proposed invariant architecture in a realistic scientific-image setting, we use the NFFA-Europe SEM dataset~\citep{nffa_europe}, ``a large-scale collection of scanning electron microscopy (SEM) images acquired at CNR-IOM (Trieste, Italy) within the NFFA-Europe project. The dataset contains approximately 25k grayscale SEM images organized into 10 semantic categories, including biological samples, nanowires, MEMS structures, patterned materials, and other nanoscale structures. Classification labels were validated by groups of nanoscientists through the online annotation platform \url{http://sem-classifier.nffa.eu}, and only images with strong annotation agreement were retained in the final dataset.''~\citep{nffa_europe}

 The dataset therefore provides a challenging benchmark for representation learning with a limited number of samples and substantial intra-class variability. Images are resized and center-cropped to a fixed spatial resolution before training. The dataset is particularly well-suited for studying rotational invariance. Object orientation in SEM imaging is typically arbitrary and does not carry semantic meaning, while many classes exhibit strong rotational or texture-like symmetries. At the same time, the dataset remains relatively small compared to standard natural-image benchmarks and exhibits substantial class imbalance, making it a suitable testbed for evaluating sample efficiency and inductive biases.

Fig.~\ref{fig:nffa_counts} shows the distribution of images across categories. While some classes contain several thousand samples, others are comparatively scarce, highlighting the importance of architectures that can generalize effectively in low-data regimes.

\section{Additional Quantitive Study}
\label{appendix:quantitative}
In addition to $E(2)$-CNNs, we compare against several widely used equivariant, non-equivariant architectures, including Harmonic Networks (H-Nets)~\citep{worrall2017}, ResNets~\citep{he2016}, and Vision Transformers~\citep{dosovitskiy2021}. These baselines were selected because they represent complementary approaches to incorporating inductive biases while remaining the closest standard architectures to our setting. To ensure a fair comparison, all models were evaluated in compact, low-parameter regimes comparable to our architecture size. For H-Nets, increasing model scale did not yield measurable accuracy improvements, and we therefore report results using a smaller configuration. In addition, scaling to larger benchmarks was computationally challenging, with runtimes often exceeding those of our method by a factor of $5$--$20$; consequently, we omit these experiments from the main comparison table. Overall, the comparisons highlight that our invariant moment architecture achieves strong performance while maintaining substantially lower parameter counts than equivariant and standard non-equivariant alternatives.
\begin{table}[h]
    \centering
    \caption{\textbf{Quantitative comparison of our \ac{CNN}} to others. Entries are given in the format \mbox{$\mathrm{Acc.} \pm \mathrm{\scriptstyle{Std.Dev.}} \scriptstyle{\mathrm{(Param.Count)}}$} } 
    \label{tab:comparison}
    \begin{tabular}{l|c|c|c|c|c}
\toprule
Dataset 
& Ours 
& E2-CNNs
& ResNet 
& H-Nets
& ViT \\
\midrule

MNIST\tiny{RotTest} 
& $\mathbf{99.42}{\scriptstyle \pm.07}$ {\tiny (119k)} 
& $97.20{\scriptstyle \pm.14}$ {\tiny (362k)} 
& $46.28{\scriptstyle \pm .26}$ {\tiny (373k)} 
& $99.05{\scriptstyle \pm .03}$ {\tiny (33k)} 
& $42.08{\scriptstyle \pm .64}$ {\tiny (853k)} \\

{\tiny rot}MNIST 
& $\mathbf{99.08}{\scriptstyle \pm .04}$ {\tiny (119k)} 
& $98.82{\scriptstyle \pm .01}$ {\tiny (362k)} 
& $98.55{\scriptstyle \pm .01}$ {\tiny (373k)} 
& $98.74 {\scriptstyle \pm .03}$ {\tiny (33k)} 
& $88.72{\scriptstyle \pm .99}$ {\tiny (853k)}\\

RESISC45 
& $\mathbf{85.77}{\scriptstyle \pm.47}$ {\tiny (230k)} 
& $84.02{\scriptstyle \pm.27}$ {\tiny (365k)} 
& $84.32 {\scriptstyle \pm.11}$ {\tiny (377k)}
& $53.06{\scriptstyle \pm .02}$ {\tiny (37k)}
& $63.18{\scriptstyle \pm .55}$ {\tiny (928k)} \\

EuroSAT\tiny{rgb}
& $\mathbf{95.29}{\scriptstyle \pm .40}$ {\tiny (149k)} 
& $93.90{\scriptstyle \pm .34}$ {\tiny (363k)} 
& $79.50{\scriptstyle \pm .31}$ {\tiny (377k)} 
& $86.28{\scriptstyle \pm 1.13}$ {\tiny (45k)}
& $76.74 {\scriptstyle \pm.46}$ {\tiny (924k)}\\

DTD$^\ast$  
& $67.14{\scriptstyle \pm.27}$ {\tiny (336k)}
& $67.18{\scriptstyle \pm .81}$ {\tiny (469k)} 
& $56.66{\scriptstyle \pm .23}$ {\tiny (377k)} 
& $38.24{\scriptstyle \pm .87}$ {\tiny (36k)}
& $37.95{\scriptstyle \pm .45}$ {\tiny (928k)} \\

NFFA
& $87.02 {\scriptstyle \pm.87}$ {\tiny (227k)}
& $87.32 {\scriptstyle \pm.30}$ {\tiny (363k)}  
& $85.08{\scriptstyle \pm .79}$ {\tiny (11M)} 
& --
& $74.20{\scriptstyle \pm .63}$ {\tiny (875k)} \\

Galaxy Zoo 
& $83.30{\scriptstyle \pm.03}$ {\tiny (227k)} 
& $\mathbf{84.39}{\scriptstyle \pm.08}$ {\tiny (363k)} 
& $82.01{\scriptstyle \pm .07}$ {\tiny (373k)}
& --
& $77.86{\scriptstyle \pm .12}$ {\tiny (924k)}\\

Col. Hist. 
& $\textbf{92.30}{\scriptstyle \pm.40}$ {\tiny (228k)} 
& $91.58{\scriptstyle \pm .41}$ {\tiny (363k)} 
& $91.36{\scriptstyle \pm .58}$ {\tiny (377k)} 
& $88.39{\scriptstyle \pm .20}$ {\tiny (90k)}
& $85.83{\scriptstyle \pm .08}$ {\tiny (915k)} \\

\bottomrule
\end{tabular}
\end{table}
\section{Minimal Implementation}
\label{appendix:implementation}
We provide a minimal implementation of the invariant feature construction used in our experiments to facilitate reproducibility. The code illustrates how invariant features are computed from moments, abstracted from framework-specific details. For the moment projection, one can directly use existing implementations of complex moments or steerable convolutions (e.g., \texttt{escnn}~\cite{cesa2022, e2cnn}). 
\newpage
\begin{lstlisting}[language=Python]
@torch.compile
def construction_of_invariants(nt: torch.Tensor,
                               nt_types: Torch.Tensor,
                               eps: float = 1e-7):
    """
    nt: Non-trivial moments (type>=1) 
        of shape [Ba, Ch, #Types, 2 (Real/Imag), H, W]
    nt_types: Types of nt [#Types]

    The first of nt is expected to be the reference type-1 [:, :, 0, ...]
    """
    mag = torch.linalg.vector_norm(nt, dim=-3)
    safe_mag = mag.clamp(min=eps)
    
    real, imag = nt.unbind(dim=-3)
    # NOTE: Torch atan2 is not gradient safe. 
    angle = atan2(imag / safe_mag, real / safe_mag, eps)

    n_angle = angle[:, :, 0:1] * -nt_types[:, None, None]
    n_mag = mag[:, :, 0:1]
    nt_mag = mag[:, :, 1:]
    nt_angle = angle[:, :, 1:]
    
    scaled_mag = magnitude_fn(n_mag, nt_mag)
    angular_inv = scaled_mag * torch.cos(nt_angle + n_angle)
    return torch.cat([mag, angular_inv],dim=2)


    
\end{lstlisting}

\section{Dimension of Quotient}
\label{appendix:dimensions}
Flusser's construction Eq.~\eqref{eq:flusser_basis} of invariant features assumes $m_1 \neq 0$, which removes rotational ambiguity. Under this condition, the group action is free on the corresponding subset of the moment space. As a consequence, the invariant representation can be interpreted as a parametrization of the quotient space whose dimension is given by $\dim(\mathcal{M}) - \dim(G)$ according to the Quotient Manifold Theorem. \todo[]{We should cite this here: \citep{dym2024}}

\textbf{Theorem (Quotient Manifold Theorem).}
Suppose $G$ is a Lie group acting smoothly, freely, and properly on a smooth manifold $M$. Then the orbit space $M / G$ is a topological manifold of dimension equal to $\dim M - \dim G$, and has a unique smooth structure with the property that the quotient map $\pi : M \to M / G$ is a smooth submersion.~\cite{lee2003}

\section{Experimental Settings}
Dataset-specific training hyperparameters, including resolutions, batch sizes, learning rates, and preprocessing steps, are summarized in Table~\ref{tab:dataset-hparams}. Shared architectural hyperparameters for both our Mom/Inv models and ESCNN baselines are provided in Table~\ref{tab:model-defaults}.
\begin{table*}[h]
\centering
\small

\begin{tabular}{lcc}
\toprule
Parameter & Optimal & ESCNN \\
\midrule
Stem kernel size & 15 & 7 \\
Stage kernels & $[11,11,11]$ & $[5,5,5]$ \\
Stage depths / blocks & $[1,2,2]$ & $[1,2,2]$ \\
Classifier hidden size & 64 & 64 \\
Dropout & 0.2 & -- \\
Maximum type & 3 (4 for DTD) & 3 (4 for DTD) \\
Block type & Mom/Inv Block & GatedBlock \\
Pooling & AveragePool & PointwiseMaxPool \\
\bottomrule
\end{tabular}
\small
\caption{Shared architecture hyperparameters.}
\label{tab:model-defaults}

\end{table*}

\begin{table}[t]
\centering

\small
\caption{Dataset-specific training hyperparameters.}
\label{tab:dataset-hparams}

\setlength{\tabcolsep}{5pt}
\renewcommand{\arraystretch}{1.1}

\begin{tabular}{lccccc}
\toprule
Dataset & Resolution & Batch  & Optimal LR & ESCNN LR & Preprocess \\
\midrule
RotMNIST
& $64^2$
& 16 / 32
& $1.525\mathrm{e}{-3}$
& $10^{-3}$
& Upscaled to 64 \\

MNIST-rot-test
& $64^2$
& 64
& $1.5\mathrm{e}{-3}$
& $10^{-3}$
& -- \\

DTD
& $96^2$
& 16
& $1.5\mathrm{e}{-3}$
& $1.5\mathrm{e}{-3}$
& Circular Mask\\

EuroSAT
& $64^2$
& 32
& $1.5\mathrm{e}{-3}$
& $10^{-3}$
&  -- \\

Galaxy Zoo
& $96^2$
& 64
& $1.5\mathrm{e}{-3}$
& $1.5\mathrm{e}{-3}$
& Circular Mask \\

RESISC45
& $96^2$
& 8
& $1.5\mathrm{e}{-3}$
& $1.5\mathrm{e}{-3}$
& --\\

NFFA
& $224^2$
& 64
& $10^{-3}$
& $1.525\mathrm{e}{-3}$
& Grayscale, Aug. Flips \\
\bottomrule
\end{tabular}

\end{table}

\section{Computational Resources}
All experiments were conducted using a single GPU due to the compact size of the proposed models. Training was performed on NVIDIA A100 GPUs on a shared compute cluster, as well as NVIDIA RTX 5090 GPUs for local experimentation and development. All models were implemented in PyTorch.



\end{document}